\documentclass[12pt,draftclsnofoot,onecolumn]{IEEEtran}
\pdfoutput=1
\usepackage{amsmath}
\usepackage[backref=page]{hyperref}
\usepackage{graphicx}
\usepackage{color}

\usepackage{cite}
\usepackage{amsthm}
\usepackage{amssymb,amsfonts}
\usepackage{algorithmic}
\usepackage{textcomp}
\usepackage{float}
\usepackage[mathscr]{eucal}
\usepackage{ifpdf}
\usepackage{units}
\usepackage{setspace}
\usepackage{algorithm}
\usepackage{array}
\usepackage{xcolor}
\usepackage{multirow}
\usepackage{enumerate}
\usepackage{mathtools}
\usepackage{footnote}
\newtheorem{theorem}{Theorem}
\newtheorem{lemma}{Lemma}

\usepackage{graphics}
\usepackage{bbm}

\usepackage{tikz}
\usetikzlibrary{arrows,calc}
\DeclareMathOperator*{\argmaxA}{arg\,max}
\DeclareMathOperator*{\argminA}{arg\,min}
\usepackage{scalerel}
\usepackage{subcaption}
\usepackage{wrapfig,booktabs}

\newcommand{\norm}[1]{\left\lVert#1\right\rVert}
\allowdisplaybreaks
\begin{document}
%-------------------------------------> Title
\title{Bayesian Non-stationary Linear Bandits for Large-Scale Recommender Systems
% \thanks{This research was supported by Grant 01IS20051 from the German Federal Ministry of Education and Research (BMBF).}
}
%-------------------------------------> Author
\author{
\IEEEauthorblockN{Saeed Ghoorchian, Evgenii Kortukov, and Setareh Maghsudi}
% First Author\inst{1}\orcidID{0000-1111-2222-3333} \and
% Third Author\inst{3}\orcidID{2222--3333-4444-5555}
\thanks{The authors are with the Faculty of Mathematics and Natural Sciences, T{\"u}bingen University, 72074 T{\"u}bingen, Germany. S. M. is also with the Fraunhofer Heinrich Herz Institute, Berlin, Germany. 
E-mail: saeed.ghoorchian@uni-tuebingen.de, evgenii.kortukov@student.uni-tuebingen.de, setareh.maghsudi@uni-tuebingen.de
}
}
%
% \authorrunning{S. Ghoorchian et al.}
% \institute{
% Eberhard Karls University of T{\"u}bingen, T{\"u}bingen, Germany
% \email{saeed.ghoorchian@uni-tuebingen.de} \\
% \email{evgenii.kortukov@student.uni-tuebingen.de} \\
% \email{setareh.maghsudi@uni-tuebingen.de}
% Princeton University, Princeton NJ 08544, USA \and
\maketitle % typeset the header of the contribution
%
%-------------------------> Abstract
\begin{abstract}
Taking advantage of contextual information can potentially boost the performance of recommender systems. In the era of big data, such side information often has several dimensions. Thus, developing decision-making algorithms to cope with such a high-dimensional context in real time is essential. That is specifically challenging when the decision-maker has a variety of items to recommend. In addition, changes in items' popularity or users' preferences can hinder the performance of the deployed recommender system due to a lack of robustness to distribution shifts in the environment. In this paper, we build upon the linear contextual multi-armed bandit framework to address this problem. We develop a decision-making policy for a linear bandit problem with high-dimensional feature vectors, a large set of arms, and non-stationary reward-generating processes. Our Thompson sampling-based policy reduces the dimension of feature vectors using random projection and uses exponentially increasing weights to decrease the influence of past observations with time. Our proposed recommender system employs this policy to learn the users' item preferences online while minimizing runtime. We prove a regret bound that scales as a factor of the reduced dimension instead of the original one. To evaluate our proposed recommender system numerically, we apply it to three real-world datasets. The theoretical and numerical results demonstrate the effectiveness of our proposed algorithm in making a trade-off between computational complexity and regret performance compared to the state-of-the-art.
% \keywords{Recommender systems \and decision-making \and multi-armed bandit \and non-stationary environment \and online learning.}
\end{abstract}
%-------------------------------------> Keywords
% \begin{IEEEkeywords}

\maketitle
%-------------------------> Section Introduction
\section{Introduction}
\label{sec:intro}
Over the past decade, recommender systems have benefited the economy by guiding decision-makers in different roles, such as service providers, consumers, and producers, toward cost-effective and time-saving actions while retaining the constraints, such as safety, privacy, and quality-of-service satisfaction. Famous examples of success stories include the recommendation systems deployed in online shopping or streaming websites that provide personalized suggestions to the users \cite{Resnick77:RSS, Adomavicius05:TTN, Hu21:RAP}. A widely-used metric to evaluate a recommender system is the returned payoff, measured in terms of the users' responses to recommended items. One well-known example is the Click-Through Rate (CTR). Therefore, the decision-making algorithms driving a recommender system aim at maximizing the payoffs over time \cite{Song16:OLL, Louedec15:AMP, Chi20:PRS}.

Due to the growing demand for online services, recommender systems must serve a large and diverse group of users by providing fast and accurate recommendations from a vast set of available items. To deliver real-time services that match the users' interests, recommender systems take advantage of side information. Thus, building efficient recommender systems becomes challenging in a large-scale scenario with high-dimensional side information and various items \cite{Song16:OLL}. In addition, online recommender systems often face distribution drifts in the environment where they are deployed. For instance, in personalized news recommendations, customer preferences over news can change over time and exhibit various seasonality patterns \cite{Wu19:DEC}. Hence, building robust recommender systems poses a significant challenge due to environmental changes. As the user's interests in items evolve, a learning agent must constantly adapt its decision-making strategy to comply faster with the environmental changes while attempting to keep the runtime as low as possible \cite{Vinagre15:AOO}. Hence, it is imperative to design adaptive and efficient algorithms, in contrast to the traditional offline models where the recommendation engine has to restart the learning from scratch regularly \cite{Al-Ghossein22:ASO, Guo17:DAF}.

In this paper, we take advantage of an online framework, namely Multi-Armed Bandit (MAB) \cite{Robbins52:SAS}, to build a recommender system and address the efficiency and robustness challenges mentioned above. The seminal MAB problem portrays a finite set of arms and a player. The player sequentially pulls one arm at each decision-making round. Upon pulling an arm, the player receives a random reward produced by an unknown generating process. The goal is to maximize the total accumulated reward over a finite time horizon. The Contextual Multi-Armed Bandit (CMAB) problem is one of the extensions of the seminal MAB problem \cite{Maghsudi17:MAB}. In the CMAB framework, each arm associates with a context vector. At each round of decision-making, the player observes these contexts before selecting an arm.

We consider a CMAB problem with high-dimensional context vectors and a large number of arms whose associated rewards follow a non-stationary linear model; the unknown model parameter can vary in time. The state-of-the-art methods that address such a problem \cite{Abbasi11:IAF,Chu11:CBW,Agrawal13:TSF, Yu17:CBW, Baekjin20:REF} either suffer from excessive computational complexity and weak regret performance, e.g., their regret bound scales as a factor of the context vectors' dimension, or do not take into account the non-stationarity of the environment. To address these shortcomings, we propose a Thompson sampling (TS)-based policy that uses Random Projection (RP) to perform dimensionality reduction, as it is computationally efficient \cite{Fodor02:ASO,Zhang16:ASL,Fern03:RPF}. In addition, our algorithm uses weighted least-squares as an efficient method to estimate the reduced model parameter while gradually forgetting past interactions. 
% Moreover, TS \cite{Thompson33:OTL} is particularly beneficial in the CMAB problems with infinite arms \cite{Agrawal13:TSF}, thereby making our policy suitable for scenarios with a large number of arms. 
Our proposed algorithm guarantees an upper regret bound that depends on the reduced dimension instead of the original dimension of context vectors. 
% Thus, our policy is specifically suitable for the case of CMAB problems with high-dimensional context vectors. 
We use three real-world datasets to evaluate our proposed recommender system. Numerical results demonstrate the efficacy of our proposed algorithm in making a trade-off between computational complexity and regret performance in non-stationary environments compared to the state-of-the-art.

%-------------------------> Subsection Organization
% \subsection{Organization}
% \label{subsec:org}
%
% The rest of the paper is as follows. In Section \ref{sec:probfor}, we formulate the problem.
In the following, we present the problem setting and notations. We then compare our work with state-of-the-art. In Section \ref{sec:decisionmaking}, we propose our decision-making strategy and introduce our algorithm, namely \texttt{D-LinTS-RP}. Section \ref{sec:regretanalysis} includes the theoretical analysis of the regret performance of D-LinTS-RP. Section \ref{sec:numanalysis} is dedicated to numerical evaluation. Section \ref{sec:conclusion} concludes the paper.
%-------------------------> Section Prob Formulation
\subsection{Problem Setting and Notations}
\label{sec:probfor}
We denote the set of arms by $\mathcal{A} = \{1, 2, \dots, A\}$. For each arm $a \in \mathcal{A}$, $\boldsymbol{x}_{a,t} \in \mathbbm{R}^{n}$ represents its corresponding random context vector at time $t$. 
% \in [0,1]
Let $r_{a,t}$, $\forall a \in \mathcal{A}, \forall t$, represent the random reward corresponding to the arm $a$ at time $t$. The instantaneous rewards of each arm $a$ at each time $t$ are independent random variables drawn from an unknown probability distribution. In this paper, we consider a non-stationary linear bandit model; that is, the reward $r_{a,t}$ for each arm $a \in \mathcal{A}$ is linear with respect to the context vector $\boldsymbol{x}_{a,t}$, and there exists an unknown time-varying parameter vector $\boldsymbol{\theta}_{t}^{\ast} \in \mathbbm{R}^{n}$ such that
\begin{equation}
\label{eq:reward}
r_{a,t} = \boldsymbol{x}_{a,t}^{\top} \boldsymbol{\theta}_{t}^{\ast} + \eta_{t},
\end{equation}
where $\eta_{t}$ is a conditionally $R$-subGaussian zero-mean random noise, where $R \geq 0$ is a fixed constant. 
% zero-mean random noise
% As it is conventional \cite{Abbasi11:IAF, Agrawal13:TSF, Yu17:CBW}, we assume that $\eta_{t}$ is conditionally $R$-sub-Gaussian, where $R \geq 0$ is a fixed constant. 
% Formally,
% %
% \begin{align}
%     \mathbbm{E}[e^{\lambda \eta_{t}} ~|~ \mathcal{F}_{t-1}] \leq e^{\frac{\lambda^{2} R^{2}}{2}}, \hspace{10mm} \forall \lambda \in \mathbbm{R}.
% \end{align}
% %
We assume that $\norm{\boldsymbol{x}_{a,t}}_{2} \leq 1$, $\forall a \in \mathcal{A}$, and $\norm{\boldsymbol{\theta}_{t}^{\ast}}_{2} \leq 1$. Therefore, $\lvert \langle \boldsymbol{\theta}_{t}^{\ast}, \boldsymbol{x}_{a,t} \rangle \rvert \leq 1$. 
% Consider the $\sigma$-algebra $\mathcal{F}_{t-1} = \sigma(\{x_{a,\tau}\}_{\tau = 1}^{t}, \forall a \in \mathcal{A}, \{\eta_{\tau}\}_{\tau = 1}^{t-1})$. 

% Alternatively, the player aims at minimizing its regret, defined as the difference between the cumulative reward of the optimal policy and the cumulative reward of the applied policy. Formally, the expected dynamic regret is defined as
The agent's goal is to maximize its total accumulated reward over a finite time horizon $T$. Alternatively, the agent aims to minimize the expected dynamic regret, defined as
\begin{equation}
\label{eq:expectedregret}
 \mathbbm{E}\left[\mathcal{R}(T)\right] = \mathbbm{E}\left[\sum_{t = 1}^{T} \left[  \boldsymbol{x}_{a_{t}^{\ast},t}^{\top} \boldsymbol{\theta}_{t}^{\ast} -  \boldsymbol{x}_{a_{t},t}^{\top} \boldsymbol{\theta}_{t}^{\ast} \right]\right],
\end{equation}
where $a^{\ast}_{t} = \argmaxA\limits_{a \in \mathcal{A}} \boldsymbol{x}_{a,t}^{\top} \boldsymbol{\theta}_{t}^{\ast}$ is the optimal arm at time $t$, and $a_{t}$ denotes the played arm at time $t$ under the applied policy.

By $\boldsymbol{I}_{d \times d}$ and $\boldsymbol{0}_{d}$, we denote an identity matrix of size $d \times d$ and a zero vector of dimension $d$, respectively. $\mu_{\min}(\boldsymbol{Z})$ represents the minimum eigenvalue of a positive definite matrix $\boldsymbol{Z}$. Moreover, for a positive definite matrix $\boldsymbol{Z} \in \mathbbm{R}^{d \times d}$ and any vector $\boldsymbol{y} \in \mathbbm{R}^{d}$, we define the norm $\norm{\boldsymbol{y}}_{\boldsymbol{Z}} = \sqrt{\boldsymbol{y}^{\top} \boldsymbol{Z} \boldsymbol{y}}$.
%-------------------------> Remark 1
% \begin{remark}
% \label{remark:1}
%
% Based on our definition of reward in (\ref{eq:reward}), we have \textcolor{blue}{$\mathbbm{E}[r_{a,t} | \mathcal{F}_{t-1}] = \boldsymbol{x}_{a,t}^{\top} \boldsymbol{\theta}_{t}^{\ast}$.} Therefore, the regret in (\ref{eq:expectedregret}) is well-defined considering $\mathcal{F}_{t-1}$. Moreover, this definition coincides with the pseudo-regret defined in \cite{Abbasi11:IAF} and \cite{Bubeck12:RAO}.
% \end{remark}
%-------------------------
%-------------------------> Subsection Related Works
\subsection{Related Works}
\label{subsec:relatedwork}
Online methods such as reinforcement learning and multi-armed bandit algorithms are popular bases to design recommender systems. Some examples include \cite{Shani05:AMB,Munemasa18:DRL,Atan23:DDO,Lihong10:ACB,Deshpande12:LBI,Ghoorchian23:OLC}. The core concept is to design algorithms that balance exploration and exploitation to maximize the total payoff over time. In the context of recommender systems, exploration means learning the payoff of new items by recommending those items to users. Exploitation involves recommending the best item to users using the collected data. Besides exploration-exploitation balance, another important criterion is to maximize the total reward while keeping the runtime as low as possible. That results in faster services, and thereby a higher users' satisfaction level.

% Specifically, 
The contextual bandit framework serves as a conventional model to formalize and solve recommendation problems. Some recent works include \cite{Tang13:AAF,Deshpande12:LBI,Mahadik20:FDB, Korda16:DCO,Li19:IAO,Li16:CFB}. Despite being designed to solve large-scale problems, the performance of the state-of-the-art methods depends strongly on the number of items and the dimension of the context vectors. For instance, in \cite{Deshpande12:LBI}, the authors consider the linear contextual bandit problem and propose the \texttt{BallExplore} algorithm to model and solve a recommendation problem with high-dimensional context vectors. They prove a regret bound that is proportional to the original dimension of context vectors. Also, the proposed algorithm runs in quadratic time regarding the original dimension $n$. As another example, in \cite{Lihong10:ACB}, the authors propose an algorithm for personalized news article recommendation that also runs in quadratic time w.r.t the original dimension $n$. In contrast, the time complexity of our proposed algorithm is linear concerning the original dimension of contextual data. 

Other recent works investigate the high-dimensional CMAB problem. Some of these approaches achieve significant improvement for the regret bounds; nonetheless, they require additional knowledge or assumptions about the characteristics of the context vectors. For example, in \cite{Abbasi12:OTC}, the authors consider a sparse linear bandit problem and propose an Upper Confidence Bound (UCB)-based policy that uses the algorithm developed in \cite{Gerchinovitz13:SRB} as a subroutine. They establish an upper regret bound of order $\tilde{O}(\sqrt{n S T})$, where $S$ is the maximum number of non-zero components in the context vector. Furthermore, the authors in \cite{Carpentier12:BTM} propose a policy which achieves a regret bound of order $O(S \sqrt{T})$. The development and the analysis are based on the assumption that the set of context vectors is the unit ball in $\mathbbm{R}^{n}$. In comparison with the research works described above, the authors in \cite{Bastani19:ODM} make several additional assumptions, e.g., on the expected covariance matrix of the samples and on the distribution of the context vectors. In return, their proposed policy achieves a regret bound of order $O(S^{2}[\log{(T)} + \log{(n)}]^{2})$. Besides, in \cite{Kim19:DRL}, the authors study the high-dimensional linear contextual bandit problem assuming that the set of contexts are sparse; i.e., only a subset of contexts is correlated with the reward. The proposed algorithm achieves a regret bound that scales logarithmically with the original dimension $n$. 
% Reference \cite{Kuzborskij19:ELB} proposes two policies that use matrix sketching to boost the computational efficiency of the algorithms developed in \cite{Abbasi11:IAF} and \cite{Agrawal13:TSF}. However, the achieved regret bounds directly depend on the original dimension $n$. 
Further, \cite{Bouneffouf17:CAB} uses a combinatorial bandit algorithm as a subroutine to select $d$ entries of context vectors out of $n$, thereby reducing the dimension of each context vector at each decision-making round. The reduced context vectors are used to update the posterior distribution on the reward parameter. This work does not provide any theoretical analysis for the regret bound.

% where $A$ represents the number of arms.
% In the following, we first review notably-related research works and highlight the novelty of our approach.
Our work extends the state-of-the-art research in the area of contextual bandits. In the following, we first review notably-related research works on stationary bandits and highlight the novelty of our approach. We then continue with reviewing the related works on non-stationary bandits. In \cite{Chu11:CBW}, the authors study the CMAB problem with linear payoffs. They propose a UCB-based algorithm, namely \texttt{LinUCB}, that achieves a regret bound of order $O(\sqrt{T n \ln^{3}(A T \frac{\ln(T)}{\delta})})$. Likewise, in \cite{Auer02:UCB}, the authors develop the decision-making policy \texttt{LinRel} that achieves a regret bound similar to that in \cite{Chu11:CBW}. Reference \cite{Abbasi11:IAF} proposes \texttt{OFUL}, a UCB-based algorithm that achieves a regret bound of order $O(n \log(T) \sqrt{T} + \sqrt{n T \log(\frac{T}{\delta})})$. In \cite{Agrawal13:TSF}, the authors utilize Thompson sampling to develop \texttt{LinTS} algorithm with a regret bound of order $O(n \sqrt{T} (\min\{\sqrt{n}, \sqrt{\ln(A)}\}) (\ln{(T)} + \sqrt{\ln{(T)}\ln{(\frac{1}{\delta})}}))$. In \cite{Yu17:CBW}, the authors propose a UCB-based algorithm \texttt{CBRAP} by using the random projection in combination with a UCB-based algorithm developed in \cite{Abbasi11:IAF}. The aforementioned algorithms either are not suitable for large-scale problems, i.e., they show poor regret or runtime performance in large-scale scenarios, or do not take into account the non-stationarity of the environment.
% a regret bound dependent on the original dimension
% Compared to the aforementioned works, the regret bound of our algorithm scales as a factor of a reduced dimension $d < n$.
% However, similar to \texttt{LinUCB}, \texttt{LinRel}, and \texttt{OFUL}, \texttt{CBRAP} is not efficient when the number of arms $A$ is large or possibly infinite.

% n real-world recommender systems user preferences evolve over time.
% They propose \texttt{SW-UCB} - an algorithm based on the "optimism in the face of uncertainty" principle that uses a sliding window to estimate the unknown parameter.
% The authors propose \texttt{D-LinUCB} - an extension of \texttt{LinUCB} that uses exponentially discounted weights to gradually forget past observations.
% that are far back in the past. 

% Real-world problems are non-stationary in most cases.
% Real-world problems frequently appear in non-stationary environments. 
% Non-stationary multi-armed bandits have attracted intensive attention in the past years \cite{Ghoorchian21:MAB, }. Potential application domains include online recommender systems \cite{Zeng16:OCA, Wu18:LCB, Wu19:DEC, Xu20:CBB}, edge computing problems \cite{Ghoorchian21:MAB}, hyperparameter optimization \cite{Lu22:NSC}, virtual reality for rehabilitation \cite{Kamikokuryo22:AAA}, split liver transplantation allocation \cite{Tang21:MAB}, evaluation of information retrieval systems \cite{Losada17:MAB}, or Covid-19 spread analysis \cite{}.

Real-world recommender systems often serve users whose preferences evolve over time. A recent line of research on linear contextual bandits is devoted to designing algorithms capable of handling this non-stationarity in the environments \cite{Cheung19:LTO, Russac19:WLB, Zhao20:ASA, Baekjin20:REF}. For example, in \cite{Cheung19:LTO}, the authors study linear stochastic bandit in a drifting environment. They propose \texttt{SW-UCB} algorithm that uses a sliding window to estimate the unknown parameter of the linear bandit and achieves a regret bound of order $\tilde{O}(n^{2/3}(B_T + 1)^{1/3}T^{2/3})$, where $B_T$ is the variation budget on the unknown parameter vector. Reference \cite{Russac19:WLB} examines the same problem in both slowly-varying and abruptly-changing environments. The authors propose \texttt{D-LinUCB}, a UCB-based algorithm that uses exponentially increasing weights to gradually forget past observations and achieves a regret bound of order $O(n^{2/3}B_{T}^{1/3}T^{2/3})$. In \cite{Zhao20:ASA}, the authors show that a simple strategy based on periodically restarting a UCB-style algorithm is sufficient to achieve the same performance in terms of regret. In \cite{Baekjin20:REF}, two perturbation approaches based on LinUCB and LinTS algorithms are developed to address the non-stationary stochastic linear bandit problem. The proposed algorithms, namely \texttt{D-RandLinUCB} and \texttt{D-LinTS}, achieve regret bounds of order $O(n^{2/3}B_{T}^{1/3}T^{2/3})$ and $O(n^{2/3}(\log{A})^{1/3}B_{T}^{1/3}T^{2/3})$, respectively. 
% \texttt{D-LinTS} is closely related to our proposed policy.  
% However, t
The aforementioned algorithms rely on the original context vectors; thus, they suffer high computational costs in large-scale scenarios. In contrast, our algorithm enjoys a regret bound that scales as a factor of a reduced dimension $d < n$ while it adapts to drifts in the environment.
% The regret bounds of the existing algorithms designed for non-stationary environments directly depend on the original context dimension $n$.

In the following section, we describe our proposed decision-making strategy to minimize the expected dynamic regret defined in (\ref{eq:expectedregret}).
 % and propose our algorithm, \texttt{D-LinTS-RP},
% under the condition that for each arm $a \in \mathcal{A}$, the player can only observe the $d$-dimensional vector $\boldsymbol{z}_{a,t}$ instead of the full context vector $\boldsymbol{x}_{a,t}$.
%-------------------------> Section D-LinTS-RP
\section{Decision-Making Strategy}
\label{sec:decisionmaking}
%
% Our proposed decision-making strategy maps the context vectors into space $\mathbbm{R}^{d}$ with a lower dimension $d < n$, thereby reducing the complexity as we update the model parameters.
% as the player interacts with the environment.
As mentioned before, the agent's goal is to minimize the expected dynamic regret (\ref{eq:expectedregret}) via learning the unknown model parameter $\boldsymbol{\theta}_{t}^{\ast}$ from history up to
time $t-1$, $\mathcal{H}_{t-1} = \{\boldsymbol{x}_{a_{\tau},\tau}, r_{a_{\tau},\tau}\}_{\tau = 1}^{t-1}$. As discussed above, working with high-dimensional data points affects the runtime and regret performance of bandit algorithms. Our proposed decision-making strategy alleviates this effect by reducing the dimension of each context vector using the RP method. More specifically, we project the data points in the original space $\mathbbm{R}^{n}$ to a random lower-dimensional space $\mathbbm{R}^{d}$, $d < n$, using a randomly designed projection matrix $\boldsymbol{P} \in \mathbbm{R}^{d \times n}$ whose columns are scaled to have unit length. It is common to design the matrix $\boldsymbol{P}$ such that each entry of $\boldsymbol{P}$ is a realization of independent and identically distributed (i.i.d.) zero-mean variables with Gaussian distribution \cite{Blum05:RPM}. Therefore, at time $t$, $\boldsymbol{z}_{a,t} = \boldsymbol{P} \boldsymbol{x}_{a,t}, \forall a \in \mathcal{A}$.
%
% \begin{align}
%     \boldsymbol{z}_{a,t} = \boldsymbol{P} \boldsymbol{x}_{a,t}, \hspace{10mm} \forall a \in \mathcal{A}.
% \end{align}
%
% At each time $t$, let $\rho_{t}$ represent the noise term caused by the random projection performed on the context vector and the parameter $\theta_{t}^{\ast}$. Formally, $\rho_{t} =  \boldsymbol{x}_{a,t}^{\top} \boldsymbol{\theta}_{t}^{\ast} -  \boldsymbol{z}_{a,t}^{\top} \boldsymbol{\psi}_{t}^{\ast},~ \forall a \in \mathcal{A}, \forall t$. In the space $\mathbbm{R}^{d}$, we have $\boldsymbol{\psi}_{t}^{\ast} = \boldsymbol{P} \boldsymbol{\theta}_{t}^{\ast}$. Thus, we can rewrite (\ref{eq:reward}) as
% %
% \begin{align}
% \label{eq:reward2}
%     r_{a,t} = \boldsymbol{z}_{a,t}^{\top} \boldsymbol{\psi}_{t}^{\ast} + \rho_{t} + \eta_{t}.
% \end{align}
% %
% Based on (\ref{eq:reward2}), we define a $\sigma$-algebra as $\mathcal{F}_{t-1}^{\prime} = \sigma(\{x_{\tau}\}_{\tau = 1}^{t}, \{\eta_{\tau}\}_{\tau = 1}^{t-1}, \{\rho_{\tau}\}_{\tau = 1}^{t-1})$. In our theoretical analysis, we use the filtration $\{\mathcal{F}_{t}^{\prime}\}_{t > 0}$ to derive the regret bound.
% To learn about unknown parameter $\boldsymbol{\psi}_{t}^{\ast}$ from history up to
% time $t − 1$, $\mathcal{H}_{t−1} = $, algorithms rely on $l^{2}$-regularized least-squares estimate of $\boldsymbol{\psi}_{t}^{\ast}$.

As we are now working in the lower-dimensional space $\mathbbm{R}^{d}$, the player's goal is to learn the unknown parameter $\boldsymbol{\psi}_{t}^{\ast} = \boldsymbol{P} \boldsymbol{\theta}_{t}^{\ast}$, from history up to
time $t-1$, $\mathcal{H}_{t-1}^{\prime} = \{\boldsymbol{z}_{a_{\tau},\tau}, r_{a_{\tau},\tau}\}_{\tau = 1}^{t-1}$. This means that, based on our model and solution, the player does not have access to the full context vectors and can only observe the $d$-dimensional vectors $\boldsymbol{z}_{a,t}$, $\forall a, t$. To learn the unknown parameter $\boldsymbol{\psi}_{t}^{\ast}$, we rely on the weighted $l^{2}$-regularized least-squares estimator with discount factor $\gamma \in (0,1)$. Formally, the estimated parameter $\hat{\boldsymbol{\psi}}_{t}$ at time $t$ is obtained as
% In the low-dimensional space $\mathbbm{R}^{d}$, the goal is to learn the unknown parameter $\boldsymbol{\psi}_{t}^{\ast}$.
%
\begin{align}
\label{eq:optimization}
    \hat{\boldsymbol{\psi}}_{t} = \argminA_{\boldsymbol{\psi} \in \mathbbm{R}^{d}} \left( \sum_{\tau=1}^{t} \gamma^{t-\tau} \left(r_{a_{t},t} - \boldsymbol{\psi}^{\top} \boldsymbol{z}_{a_{\tau},\tau} \right)^{2} + \frac{\lambda}{2} \norm{\boldsymbol{\psi}}_{2}^{2} \right).
\end{align}
where $\lambda > 0$ is a regularization parameter. At each time $t$, The closed form of the weighted least-squares estimator can be calculated as $\boldsymbol{\hat{\psi}}_{t} = \boldsymbol{Z}_{t}^{-1} \boldsymbol{b}_{t}$, where $\boldsymbol{Z}_{t} = \sum_{\tau = 1}^{t-1} \gamma^{-\tau} \boldsymbol{z}_{a_{\tau},\tau} \boldsymbol{z}_{a_{\tau},\tau}^{\top} + \lambda \gamma^{-(t-1)} \boldsymbol{I}_{d \times d}$ and $\boldsymbol{b}_{t} = \sum_{\tau = 1}^{t-1} \gamma^{-\tau} r_{a_{\tau},\tau} \boldsymbol{z}_{a_{\tau},\tau}$.
%
% \begin{align} \label{eq:eqforZ}
%     \boldsymbol{Z}_{t} &= \sum_{\tau = 1}^{t-1} \gamma^{-\tau} \boldsymbol{z}_{a_{\tau},\tau} \boldsymbol{z}_{a_{\tau},\tau}^{\top} + \lambda \gamma^{-(t-1)} \boldsymbol{I}_{d \times d},
%     \\
%     \label{eq:eqforb}
%     \boldsymbol{b}_{t} &= \sum_{\tau = 1}^{t-1} \gamma^{-\tau} r_{a_{\tau},\tau} \boldsymbol{z}_{a_{\tau},\tau}.
%     % \\ \label{eq:eqforpsi}
%     % \boldsymbol{\hat{\psi}}_{t} &= \boldsymbol{Z}_{t}^{-1} \boldsymbol{b}_{t}.
% \end{align}
%
In addition, at each time $t$, we define $\tilde{\boldsymbol{Z}}_{t} = \sum_{\tau = 1}^{t-1} \gamma^{-2 \tau} \boldsymbol{z}_{a_{\tau},\tau} \boldsymbol{z}_{a_{\tau},\tau}^{\top} + \lambda \gamma^{-2(t-1)} \boldsymbol{I}_{d \times d}$.
%
% \begin{align} 
% \label{eq:eqforZtilde}
%     \tilde{\boldsymbol{Z}}_{t} &= \sum_{\tau = 1}^{t-1} \gamma^{-2 \tau} \boldsymbol{z}_{a_{\tau},\tau} \boldsymbol{z}_{a_{\tau},\tau}^{\top} + \lambda \gamma^{-2(t-1)} \boldsymbol{I}_{d \times d}.
% \end{align}
%
%-------------------------> Section D-LinTS-RP
% \subsection{D-LinTS-RP Algorithm}
% \label{sec:D-LinTS-RP}
%

Our proposed decision-making strategy, D-LinTS-RP, is summarized in \textbf{Algorithm \ref{Alg:D-LinTS-RP}}. As mentioned before, in the initial phase, D-LinTS-RP constructs the random projection matrix $\boldsymbol{P}$ as a random matrix whose elements are drawn from a normal distribution $\mathcal{N}(0,\kappa^{2})$, where $\kappa$ is a parameter of the algorithm.
% At each time $t$, D-LinTS-RP utilizes this matrix to compute the reduced context vector $\boldsymbol{z}_{a,t}$ for each arm $a \in \mathcal{A}$.
% The idea behind 
At each time $t$, similar to \cite{Baekjin20:REF}, our algorithm perturbs the estimated parameter $\hat{\boldsymbol{\psi}}_{t}$ via a multivariate Gaussian perturbation $\boldsymbol{W} \sim  \mathcal{N}(\boldsymbol{0}_{d}, \xi^{2} \boldsymbol{I}_{d \times d})$, with $\xi$ being a  tunable parameter. 
% Note that, this is equivalent to perturbing historical rewards using a univariate Gaussian perturbation due to the linear invariance property.
Afterward, D-LinTS-RP calculates the perturbed estimate $\boldsymbol{\tilde{\psi}}_{t}$ and selects the arm that has the highest value of $\tilde{r}_{a,t} = \boldsymbol{\tilde{\psi}}_{t}^{\top} \boldsymbol{z}_{a,t}$. Finally, it observes the corresponding reward value and updates the model parameter using the reduced context vector of the selected arm and the corresponding reward.

% Compared to UCB-based algorithms, our randomized algorithm 
% does not require computing $\norm{\boldsymbol{z}_{a,t}}_{\boldsymbol{Z}_{t}^{-1}}$ of all arms at each time. Hence, it 
Our randomized algorithm can be efficiently implemented when the set of arms is large. Moreover, D-LinTS-RP adapts to parameter changes by using the discount factor, thereby reducing the influence of past observations with time. The computational complexity of D-LinTS-RP is polynomial w.r.t. the lower dimension $d$. We observe that for a fixed $d$, the computational complexity of D-LinTS-RP scales linearly w.r.t. the original dimension $n$. This is an improvement over the previous methods, such as the works proposed in \cite{Agrawal13:TSF}, \cite{Chu11:CBW}, \cite{Abbasi11:IAF}, and \cite{Deshpande12:LBI}.

%-------------------------> Algorithm 1
\begin{algorithm}[t!]
\caption{D-LinTS-RP: Discounted Linear Thompson Sampling with Random Projection.
% D-LinTS-RP: Bayesian CMAB with Random Projection.
}
\label{Alg:D-LinTS-RP}
\begin{algorithmic}
\STATE \textbf{Input}: Parameters $d$, $\kappa$, $\lambda \geq 1$, $\xi > 0$, and $0< \gamma <1$. %\\ %$\delta$, $\varepsilon$
\STATE \textbf{Initialize}: $\boldsymbol{Z}_{1} = \lambda \boldsymbol{I}_{d \times d}$, $\boldsymbol{\tilde{Z}}_{1} = \lambda \boldsymbol{I}_{d \times d}$, and $\boldsymbol{b}_{1} = \boldsymbol{0}_{d}$. %, and $\boldsymbol{\hat{\psi}}_{1} = \boldsymbol{0}_{d}$.
\FOR{$i = 1, \dots, d$}
\FOR{$j = 1, \dots, n$}
\STATE Generate $g_{i,j} \sim \mathcal{N}(0,\kappa^{2})$ and assign $\boldsymbol{P}[i,j] = g_{i,j}$.
% a random Gaussian variable
\ENDFOR
\ENDFOR
\FOR{$t = 1, \dots, T$}
\STATE Calculate $\boldsymbol{\hat{\psi}}_{t} = \boldsymbol{Z}_{t}^{-1} \boldsymbol{b}_{t}$.
% \FOR{$a = 1, \dots, A$}
\STATE Observe the new context vectors $\boldsymbol{x}_{a,t}$, $\forall a \in \mathcal{A}$.
\STATE Calculate $\boldsymbol{z}_{a,t} = \boldsymbol{P} \boldsymbol{x}_{a,t}$, $\forall a \in \mathcal{A}$.
% \ENDFOR
% \STATE Sample $\boldsymbol{\tilde{\psi}}_{t}$ from $\mathcal{N}(\boldsymbol{\hat{\psi}}_{t}, \nu_{t}^{2} \boldsymbol{Z}_{t}^{-1})$.
\STATE Calculate $\boldsymbol{\tilde{\psi}}_{t} = \boldsymbol{\hat{\psi}}_{t} + \boldsymbol{Z}_{t}^{-1} \boldsymbol{\tilde{Z}}_{t}^{1/2} \boldsymbol{W}$, where $\boldsymbol{W} \sim \mathcal{N}(\boldsymbol{0}_{d}, \xi^{2} \boldsymbol{I}_{d \times d})$.
% \STATE Calculate
% %
% % \begin{align}
% % \label{eq:width}
% $
%     \tilde{r}_{a,t} = \boldsymbol{\tilde{\psi}}_{t}^{\top} \boldsymbol{z}_{a,t},~ \forall a \in \mathcal{A}.
% $
% % \end{align}
% %
\STATE Select arm %$a_{t}$ such that 
%
% \begin{align}
% \label{eq:ArmSelectionStrategy}
$
    a_{t} = \argmaxA_{a \in \mathcal{A}} \boldsymbol{\tilde{\psi}}_{t}^{\top} \boldsymbol{z}_{a,t} %\tilde{r}_{a,t} %.
$
% \end{align}
%\
% \STATE 
and observe reward $r_{a_{t},t}$.
\STATE Update $\boldsymbol{Z}_{t+1} = \gamma \boldsymbol{Z}_{t} +  \boldsymbol{z}_{a_{t},t} \boldsymbol{z}_{a_{t},t}^{\top} + (1-\gamma)\lambda \boldsymbol{I}_{d \times d}$.
\STATE Update $\boldsymbol{\tilde{Z}}_{t+1} = \gamma^{2} \boldsymbol{\tilde{Z}}_{t} +  \boldsymbol{z}_{a_{t},t} \boldsymbol{z}_{a_{t},t}^{\top} + (1-\gamma^{2})\lambda\boldsymbol{I}_{d \times d}$.
\STATE Update $\boldsymbol{b}_{t+1} = \gamma \boldsymbol{b}_{t} + r_{a_{t},t} \boldsymbol{z}_{a_{t},t}$.
\ENDFOR
\end{algorithmic}
\end{algorithm}
%-------------------------

%-------------------------> Section Analysis of D-LinTS-RP
\section{Theoretical Analysis}
\label{sec:regretanalysis}
%
% In this section, we theoretically analyze the D-LinTS-RP algorithm by stating an upper bound on its regret. 
The following theorem states an upper bound on the expected dynamic regret of the decision policy D-LinTS-RP, summarized in Algorithm \ref{Alg:D-LinTS-RP}.
%-------------------------> Theorem 1
\begin{theorem}
\label{thm:regbound}
Let $\kappa^{2} = \frac{1}{d}$ and $B_{T} = \sum_{t=1}^{T-1} \norm{\boldsymbol{\theta}_{t}^{\ast} - \boldsymbol{\theta}_{t+1}^{\ast}}_{2}$. For any $\delta, \varepsilon \in (0,1)$ and $\lambda \geq 1$, with probability $1 - 2 \exp(- \frac{d \varepsilon^{2}}{8})$, the expected dynamic regret of D-LinTS-RP is upper bounded as %expected 
% , with probability $(1-\delta)(1 - 4 T \exp(- \frac{d \varepsilon^{2}}{8}))$
%
% \frac{T}{1-\gamma}
\begin{align}
\label{eq:regbound}
   \mathbbm{E}[\mathcal{R}(T)] = 
   O \left( T {\bigg(} \frac{\log(T) B_{T}}{1-\gamma} + \exp{(-d \varepsilon^{2})} + \varepsilon {\bigg(} 1 + \sqrt{\frac{d}{T} \log(A)} {\bigg)} {\bigg)} \right).
\end{align}
\end{theorem}
%-------------------------
\begin{proof}
% See Section III-A of supplementary material.
See Appendix \ref{app:TheoremOneProof}.
% See Section 3.1 of supplementary material.
\end{proof}
%-------------------------
The original dimension $n$ does not appear in our regret bound, which is an improvement over the previous works that directly scale with $n$. Note that although the regret bound (\ref{eq:regbound}) depends on the reduced dimension $d$, choosing a small $d$ does not necessarily reduce the regret as, in this case, the obtained regret bound holds with a low probability. Indeed, choosing $d$ to be too small might even increase the regret due to the excessive information loss. 
% In our numerical experiments, we show that by reducing the dimension of context vectors appropriately, we can achieve good regret performance compared to some benchmark algorithms. 
% On the other hand, choosing a large $d$ expands the regret bound while that holds with a higher probability; that is, a large $d$ decreases the uncertainty of the provided regret bound, as expected intuitively. 
Also, as mentioned before, choosing a smaller value of $d$ improves the running time of our proposed algorithm. Therefore, selecting a suitable value for the reduced parameter $d$ is crucial for achieving a low computational complexity while ensuring negligible regret. We elaborate on this trade-off in our numerical analysis in the next section. %Section \ref{sec:numanalysis}.
Our algorithm does not require the knowledge of the total variation budget $B_{T}$. However, as we will see in our numerical analysis, it requires a suitable reduced dimension and a tuned discount factor as the input to achieve efficient runtime and regret performance.
% Our algorithm does not require the knowledge of $B_{T}$. However, as demonstrated in our numerical analysis, it requires a suitable reduced dimension and a tuned discount factor as the input in order to perform efficiently in terms of runtime and regret.
%-------------------------
%-------------------------> Section Numerical Analysis
\section{Numerical Analysis}
\label{sec:numanalysis}
In this section, we provide more insights into the effects of high-dimensional features and environmental changes on the performance of learning algorithms. Besides, we clarify how our proposed algorithm mitigates the adverse effects on runtime and regret performance by reducing the feature dimensions and adapting to drifts, respectively. We also compare the performance of our algorithm with conventional benchmarks using three real-world datasets. In particular, we study the following issues through numerical experiments: (i) The performance of our proposed decision-making policy compared to benchmark algorithms in terms of runtime, Click-Through-Rate (CTR), and Normalized Discounted Cumulative Gain (NDCG); (ii) the effect of the reduced dimension $d$ on the performance of our algorithm; (iii) the trade-off between computational complexity and regret bound together with the balance found by our algorithm, in particular, in comparison with the theoretical results. 
% The source code for our algorithm and experiments in this paper can be accessed via the anonymized link in the footnote \footnote{Source code: \url{https://anonymous.4open.science/r/dlintsrp_anonymized/}}.
The source code for our algorithm and experiments in this paper are publicly available \footnote{Source code: \url{https://github.com/saeedghoorchian/D-LinTS-RP.git}}.
%

% \subsubsection{Baselines}
\textbf{Baselines:}
We compare our algorithm with state-of-the-art context-aware and context-agnostic algorithms. Context-aware benchmarks in our experiment can be divided into four categories. First, we consider \textbf{D-LinTS} \cite{Baekjin20:REF}, which is designed for non-stationary environments and uses the original context vector with dimension $n$ to select arms. Second, we consider \textbf{CBRAP} \cite{Yu17:CBW}, which is designed for bandit problems with high dimensions in stationary environments. Similar to our algorithm, CBRAP can reduce the dimension of original features at each time of play. As a result, we expect that they enjoy lower computational costs compared to other benchmarks. Third, we consider \textbf{LinTS} \cite{Agrawal13:TSF} that is neither designed for changing environments nor high-dimensional features. It utilizes the original context vectors with dimension $n$ to select arms in stationary environments and has a Bayesian approach similar to our algorithm. As the last context-aware benchmark, we consider \textbf{DeepFM} \cite{Guo17:DAF}, which is a state-of-the-art algorithm designed for CTR prediction, a technique widely employed when designing offline recommender systems. This is in contrast to the online nature of our proposed algorithm.
% \cite{Agrawal13:TSF}, and \textbf{LinUCB} \cite{Chu11:CBW}
% LinTS has a Bayesian approach similar to our proposed algorithm, whereas LinUCB relies on building confidence bounds to select arms.
% LinTS, D-LinTS, LinUCB, and DeepFM always observe all features. Hence, they incur a higher computational cost compared to context-agnostic benchmarks.

As the context-agnostic benchmark, we choose $\varepsilon$-\textbf{Greedy} \cite{Auer02:FTA}, which is a standard method despite its weakness due to being blind to contextual information. It does not incur a high computational cost as it does not rely on feature observations and works only based on collected rewards. In contrast, LinTS, D-LinTS, and DeepFM always observe all features. Hence, they incur a higher computational cost compared to other benchmarks. We also consider a \textbf{random} policy that selects an action uniformly at random at each time.
%-------------------------> Section Datasets
\subsection{Settings and Data Preparation}
\label{sec:Datasets}
We evaluate the performance of our algorithm using three real-world datasets. In the following, we introduce each dataset individually and explain the data preparation steps for our experiments. \textbf{Table \ref{Table:datasets}} presents a summary of the datasets used in our experiment.
%

%-------------------------> MovieLens
% \subsubsection{MovieLens \textcolor{blue}{10M}}
\textbf{MovieLens 10M:}
This dataset contains users' ratings and tag applications applied to a set of movies from the MovieLens website \cite{Harper16:TMD}. 
%\footnote{https://grouplens.org/datasets/movielens/}~
The ratings have a $5$-star scale, with half-star increments. Thus, the possible values for rating are $0.5, 1, 1.5, \dots, 5$. In our experiment, we select the top $A = 1000$ movies based on the number of ratings given by the users. We form the context vector for each user by using the movies that the user has watched together with the tags he applied to those movies. Afterward, we extract latent context vectors for each arm (movie), using a low-rank matrix factorization with $100$ latent contexts. Then, we represent the context vector of the movie-user pair by concatenating the user and the movie context vectors. The dimension of the final context vectors is $n = 120$.
% The users appear in the dataset more than once. Hence,
We generate a user stream by considering only the users that have rated any of the $A = 1000$ movies. We take the users in the order of their appearances in the data. Hence, it is possible that a specific user appears more than once in our experiment. In this case, we sort the appearances of this specific user according to the timestamps. Our experiment with MovieLens dataset contains $2,885$ unique users. At each time, one user from the user stream arrives, the environment reveals the context vectors, and the algorithm needs to recommend one of $A = 1000$ movies to the user. We employ the implicit feedback model to generate rewards for the benchmark algorithms; if there is a rating present in the dataset, this indicates user interest, and we assign a reward equal to $1$. Otherwise, the reward is $0$. Hence, the reward is $0$ for an unwatched movie.
% At each time, one user from the user stream arrives, the environment reveals $A = 1000$ context vectors, and the algorithm needs to recommend one of $A = 1000$ movies to the user.
% context vectors $A = 1000$

% \subsubsection{Jester}
\textbf{Jester:}
It consists of more than $1.7$ million joke ratings on a continuous scale from $-10$ to $10$ for $A = 140$ jokes \cite{Goldberg01:EAC}. 
%\footnote{https://goldberg.berkeley.edu/jester-data/}~
We extract latent contexts for representing the users and arms (jokes), using a low-rank matrix factorization with $150$ latent contexts. We then concatenate these context vectors to create the context vector of each joke-user pair. Hence, the dimension of the final context vectors is $n = 300$. In the Jester dataset, the time of user-item interaction is unavailable. Hence, to create a user stream, we sample users from the original dataset uniformly at random with replacement. This procedure results in $59,132$ unique users in our experiment. The algorithm recommends a joke to an incoming user and receives a reward of $1$ if the corresponding rating is greater than $0$. If the rating is less than $0$ or no rating for a joke by a user exists, then the algorithm collects a reward equal to $0$. This pre-processing step rests on an assumption that a missing rating corresponds to a user not being interested in a joke. During the creation of the Jester dataset, the jokes were shown to users sequentially. A user not rating a joke means they stopped using the Jester website before seeing it. We interpret this as the user losing interest; thereby, we disincentivize the algorithm from recommending such jokes to the user.

%-------------------------> Table
\begin{table}[!t]
\renewcommand{\arraystretch}{1.3}
\renewcommand{\tabcolsep}{1.5mm}
\begin{center}
\caption{Summary of datasets' characteristics.}
\label{Table:datasets}
\resizebox{0.6\textwidth}{!}{
\begin{tabular}{ |c|c|c|c| } 
\hline
% \multirow{2}{*}{\textbf{Dataset}} ($A$) & \textbf{Number of Arms} & \textbf{Context Dimension} & \textbf{Number of}  \\ 
% \hline
% & $A$ & $n$ & \textbf{Unique Users}  \\ 
\textbf{Dataset} & \textbf{\#Arms ($A$)} & \textbf{\#Features ($n$)}  & \textbf{\#Unique Users}  \\ 
\hline
MovieLens 10M & 1000 & 120 & 2,885 \\ 
\hline
Jester & 140 & 300 & 59,132\\ 
\hline
Amazon Books & 400 & 200 & 7,000\\ 
\hline
\end{tabular}%
}
\end{center}
\end{table}
%
%-------------------------

% \subsubsection{Amazon Books}
\textbf{Amazon Books:} 
This dataset is a subset of the $2018$ Amazon Review Data \cite{Ni19:JRU} that contains $51,311,621$ book ratings on a discrete scale from $0$ to $5$. The original data spans a period from May $1996$ to October $2018$. As the ratings in the original data are highly sparse, we use a subset of the original rating data from $15$ December $2012$ to $15$ June $2013$ to form a user stream for our experiment. From this subset of ratings, we extract $A = 400$ items that have the most reviews. Then, we pick the $7,000$ most active users (with the most number of rated books) amongst those users that have rated the $400$ items in the experiment. We sample users from this set uniformly at random with replacement to form the final user stream. We extract latent contexts for representing the users and books using a low-rank matrix factorization with latent space dimension $100$. We concatenate these
latent vectors to create the final context vector of each book-user pair with dimension $n = 200$. Similar to Movielens 10M dataset, we consider an implicit feedback recommendation system model. When the algorithm recommends an item to a user, if the rating for this user-item pair is present in the original data, the reward is $1$. Otherwise, the reward is $0$.
% If the algorithm recommends an item to a user, it will receive a reward of $1$ if the rating for this user-item pair is present in the original data, and reward $0$ otherwise.
% that is revealed to the algorithm
% Hence, the dimension of the final context vector is $n = 200$.

Using the aforementioned setup, we create two user streams for each dataset as the validation and evaluation data with $30,000$ and $100,000$ time steps, respectively. We use the validation data to tune the hyperparameters of each benchmark algorithm by performing a grid search. In Appendix \ref{app:AddInfo},
% In Section IV of the supplementary material, 
we present a detailed explanation of tuning the parameters and list the tuned parameters of algorithms used in our experiments in \textbf{Table \ref{Table:parameters}}. 
% \textbf{Table 3}
We use the training data to evaluate the algorithms on $T = 100,000$ user appearances in each user stream.
% We evaluate the algorithms on $T = 100,000$ user appearances in each user stream. 
In order to simulate a non-stationary environment during evaluation, we introduce change points at times $\{5000, 10000, 20000, 35000, 50000, 65000, 80000, 90000\}$. At every change point, we cyclically shift the arms backward by one-third of the size of the arms' set. For example, for MovieLens 10M dataset, we shift the arm indices by $333$ at each change point. This means if the algorithm chooses arm $k$ after a change point, it receives the reward it would get from arm $(k + 333 \mod 1000)$ before the change point. This ensures a piece-wise stationary expected reward for each arm throughout the experiment. This way of introducing non-stationarity can correspond to a shift in users' preferences or in items' popularity.
%Jester
% We introduce non-stationarity in the reward-generating process the same way as in the MovieLens experiment, with the same set of change points of mean rewards.
%Amazon
% We introduce non-stationarity into the data the same way as done for Movielens 10M and Jester datasets with the same set of change points.

To deploy the DeepFM in an online recommender system setting, we proceed as follows. During an initial exploration phase of length $1000$, arms are chosen uniformly at random. After the exploration phase, we re-train the model from scratch every $1000$ time steps on all the already observed context-reward pairs. After that, we use the DeepFM model to estimate the expected reward of each user-item pair using the given context vector and then choose the arm with the highest estimated reward. After the reward is revealed by the environment, it is saved in the model memory to be used later for re-training the model. This way, we deploy a CTR prediction model in our experiments such that it can make use of newly gathered data over time.
% 
% This way, we imitate how a CTR prediction model could be used in a real-time recommender system environment in order to timely make use of newly gathered data.

% \subsection{Insights}
% \textcolor{blue}{As mentioned before, we evaluate the algorithms on each dataset by} by running them for $T = 100,000$ over $5$ repetitions and report the results by averaging over the repetitions. 
% corresponding to different datasets
We run the algorithms on the evaluation data for each dataset using the aforementioned setup for $5$ repetitions and report the results by averaging over the repetitions. Random projection matrix, if used, stays the same for each repetition. We report the average runtime, the average cumulative reward, and the average cumulative NDCG@5 of each policy. For D-LinTS-RP and CBRAP, we reduce the context dimension to $5\%$, $10\%$, $20\%$, and $50\%$ of the original context dimension to analyze the effect of the reduced dimension $d$ on the algorithms' performance. For the sake of presentation, in \textbf{Table \ref{Table:results}}, we list some selected results corresponding to context-aware benchmarks and reduced dimensions equal to $20\%$ and $50\%$ of the original dimension, and defer the full results to Appendix \ref{app:AddExps}.
% Section V of the supplementary material.
All the policies are evaluated on a single compute node with 64 Intel Xeon Gold 6226R CPUs and 64G of RAM.
%-------------------------> Table
\begin{table*}[t!]
\renewcommand{\arraystretch}{1.2}
\renewcommand{\tabcolsep}{1.2mm}
    \begin{center}
    \caption{Comparison of cumulative reward, cumulative NDCG@5, and time consumption of different policies corresponding to different datasets and context dimensions. The reported values are averaged over five repetitions.}
    \vspace{2mm}
    \label{Table:results}
    \resizebox{0.98\textwidth}{!}{
    \begin{tabular}{c|c|c|c|c|c}
        \hline
        \bfseries Dataset & \bfseries Policy & \bfseries Context Dimension & \bfseries Runtime (second) & \bfseries Cumulative Reward  & \bfseries Cumulative NDCG@5 \\
        \hline 
        %-------------------------> MovieLens
        \multirow{7}{*}{MovieLens 10M} & D-LinTS & 120 & 1739.8 & 74771.6 & 48499.1 \\ 
        \cline{2-6} 
        & LinTS & 120 & 1370.9 & 53814.0 & 33572.3 \\
        \cline{2-6} 
        % \cline{3-5}
        % & \multirow{5}{*}{D-LinTS-RP} & 6 & 1027.4 & 59379.8 & 33934.6 \\
        % \cline{3-6} 
        % &  & 12 & 1054.9 & 66628.6 & 37764.2 \\
        % \cline{3-6} 
        & \multirow{2}{*}{D-LinTS-RP} & 24 & 1109.2 & 70130.0 & 39509.2 \\
        \cline{3-6} 
        & & 60 & 1231.5 & 74294.6 & 44014.6 \\
        % \cline{3-6} 
        % & & 120 & 1753.4 & 74379.8 & 48048.1 \\
        % \cline{2-6} 
        %  & LinUCB & 120 & 2452.2 & 58130.0 & 33378.2 \\
        \cline{2-6} 
        % & \multirow{5}{*}{CBRAP} & 6 & 1647.7 & 38660.0 & 23687.6 \\
        % \cline{3-6}
        % & & 12 & 1645.8 & 42485.0 & 25662.7 \\
        % \cline{3-6} 
        & \multirow{2}{*}{CBRAP} & 24 & 1917.0 & 57314.0 & 32521.0 \\
        \cline{3-6} 
        & & 60 & 1989.3 & 59425.0 & 34382.5 \\
        % \cline{3-6} 
        % & & 120 & 2440.3 & 61445.0 & 36496.0\\
        \cline{2-6} 
         & DeepFM & 120 & 4174.4 & 25419.8 & 23624.2 \\
        % \cline{2-6} 
        % & $\epsilon$-greedy & -- & 686.7 & 33066.4 & 22866.9 \\
        % \cline{2-6} 
        % & Random & -- & 857.8 & 20683.4 & 20634.6 \\
        %-------------------------> Jester
        \hline 
        \multirow{7}{*}{Jester} & D-LinTS & 300 & 6800.4 & 44134.6 & 33895.4  \\
        \cline{2-6} 
        & LinTS & 300 & 2769.7 & 36695.2 & 30340.7 \\
        \cline{2-6} 
        % & \multirow{5}{*}{D-LinTS-RP} 
        % & 15 & 596.0 & 34681.8 & 25716.7  \\
        % \cline{3-6} 
        % &  & 30 & 438.4 & 41278.6 & 32077.9 \\
        % \cline{3-6} 
        & \multirow{2}{*}{D-LinTS-RP} & 60 & 610.6 & 43382.8 & 32116.6 \\
        \cline{3-6} 
        &  & 150 & 1554.3 & 43693.2 & 33030.9 \\
        % \cline{3-6} 
        % &  & 300 & 6362.9 & 43747.8 & 34022.0 \\
        % \cline{2-6} 
        % & LinUCB & 300 & 798.8 & 30426.0 & 25540.9 \\
        \cline{2-6} 
        % & \multirow{5}{*}{CBRAP} 
        % & 15 & 450.9 & 18159.0 & 20113.8  \\
        % \cline{3-6} 
        % &  & 30 & 475.8 & 19918.0 & 18023.0 \\
        % \cline{3-6} 
        & \multirow{2}{*}{CBRAP} & 60 & 521.3 & 26171.0 & 22018.0 \\
        \cline{3-6} 
        &  & 150 & 673.6 & 25081.0 & 22004.3 \\
        % \cline{3-6} 
        % &  & 300 & 841.0 & 27740.8 & 23026.6 \\
        \cline{2-6} 
         & DeepFM & 300 & 3912.7 & 19275.0 & 22513.9 \\
        % \cline{2-6}
        % & $\epsilon$-greedy & -- & 223.7 & 29832.4 & 27273.1 \\
        % \cline{2-6} 
        % & Random & -- & 235.5 & 13929.0 & 18179.2 \\
         %-------------------------> Amazon Books
        \hline 
        \multirow{7}{*}{Amazon Books} & D-LinTS & 200 & 2549.3 & 35018.2 & 12929.0 \\
        \cline{2-6} 
        & LinTS & 200 & 1315.6 & 9513.6 & 4594.4\\
        \cline{2-6} 
        % & \multirow{5}{*}{D-LinTS-RP} 
        % & 10 & 505.8 & 20989.6 & 8057.9  \\
        % \cline{3-6} 
        % &  & 20 & 537.8 & 26544.0 & 9694.7 \\
        % \cline{3-6} 
        & \multirow{2}{*}{D-LinTS-RP}  & 40 & 586.2 & 34118.0 & 12257.4 \\
        \cline{3-6} 
        &  & 100 & 973.2 & 34654.0 & 12603.2 \\
        % \cline{3-6} 
        % &  & 200 & 2568.8 & 34788.6 & 12995.2 \\
        % \cline{2-6} 
        % & LinUCB & 200 & 1131.1 & 5257.0 & 2914.6 \\
        \cline{2-6} 
        % & \multirow{5}{*}{CBRAP} 
        % & 10 & 742.2 & 1901.0 & 1894.8  \\
        % \cline{3-6} 
        % &  & 20 & 808.4 & 8618.0 & 3369.9 \\
        % \cline{3-6} 
        & \multirow{2}{*}{CBRAP} & 40 & 833.8 & 12500.0 & 4948.5 \\
        \cline{3-6} 
        &  & 100 & 918.5 & 5779.0 & 3073.5 \\
        % \cline{3-6} 
        % &  & 200 & 1281.2 & 7549.0 & 3263.6 \\
        \cline{2-6}
         & DeepFM & 200 & 4377.9 & 5562.8 & 3402.3 \\
        % \cline{2-6}
        % & $\epsilon$-greedy & -- & 296.9 & 1102.0 & 1454.8 \\
        % \cline{2-6} 
        % & Random & -- & 361.2 & 1625.2 & 1625.4 \\
        \hline
    \end{tabular}
    }
    \end{center}
\end{table*}
%-------------------------

%------------------------->
\begin{figure*}[t!]
    \centering
     \begin{subfigure}[t]{1\textwidth}
         \centering
         \includegraphics[width=1\textwidth]{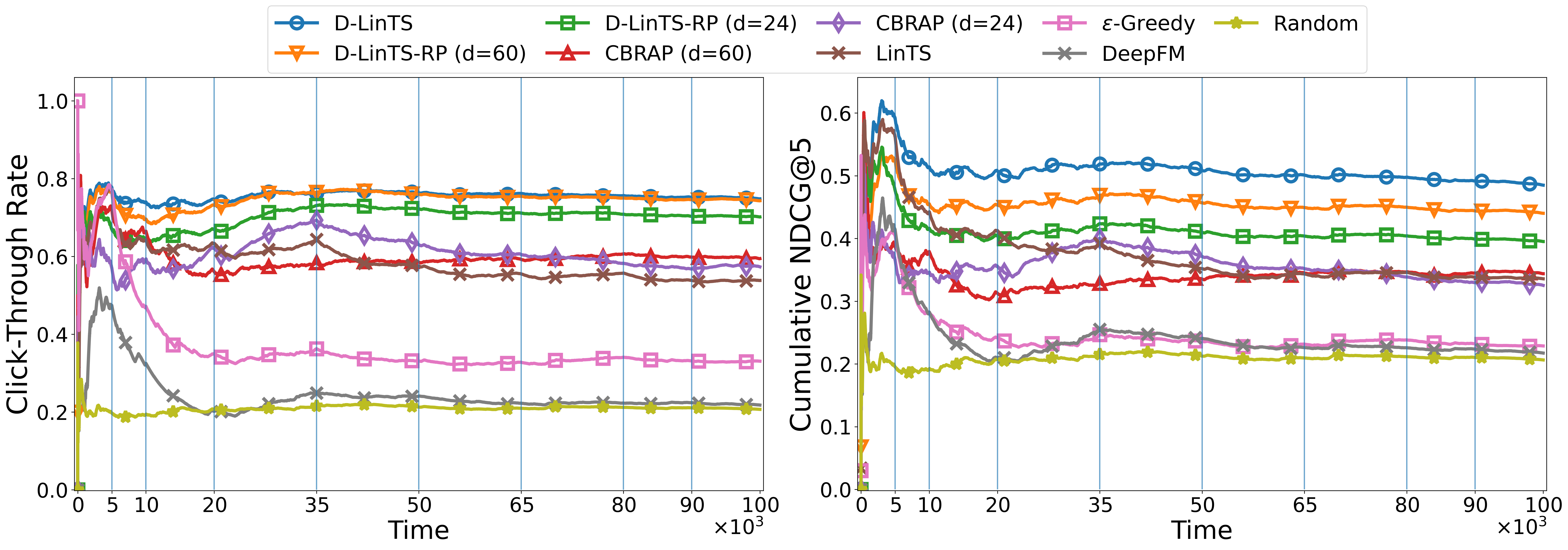} 
         \caption{Results for MovieLens 10M dataset.}
         \vspace{1mm}
         \label{subfig:movielens_ctr_ndcg}
     \end{subfigure}
     \begin{subfigure}[t]{1\textwidth}
         \centering
         \includegraphics[width=1\textwidth]{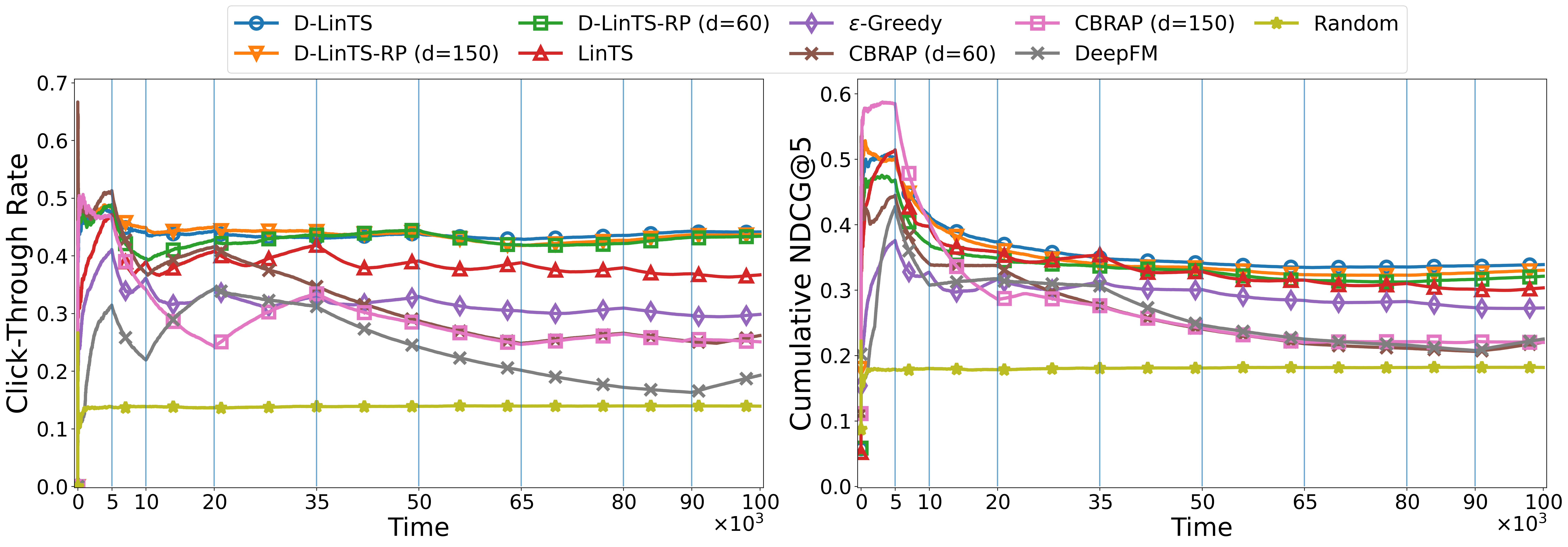} 
         \caption{Results for Jester dataset.}
         \vspace{1mm}
         \label{subfig:jester_ctr_ndcg}
     \end{subfigure}
     \begin{subfigure}[t]{1\textwidth}
         \centering
         \includegraphics[width=1\textwidth]{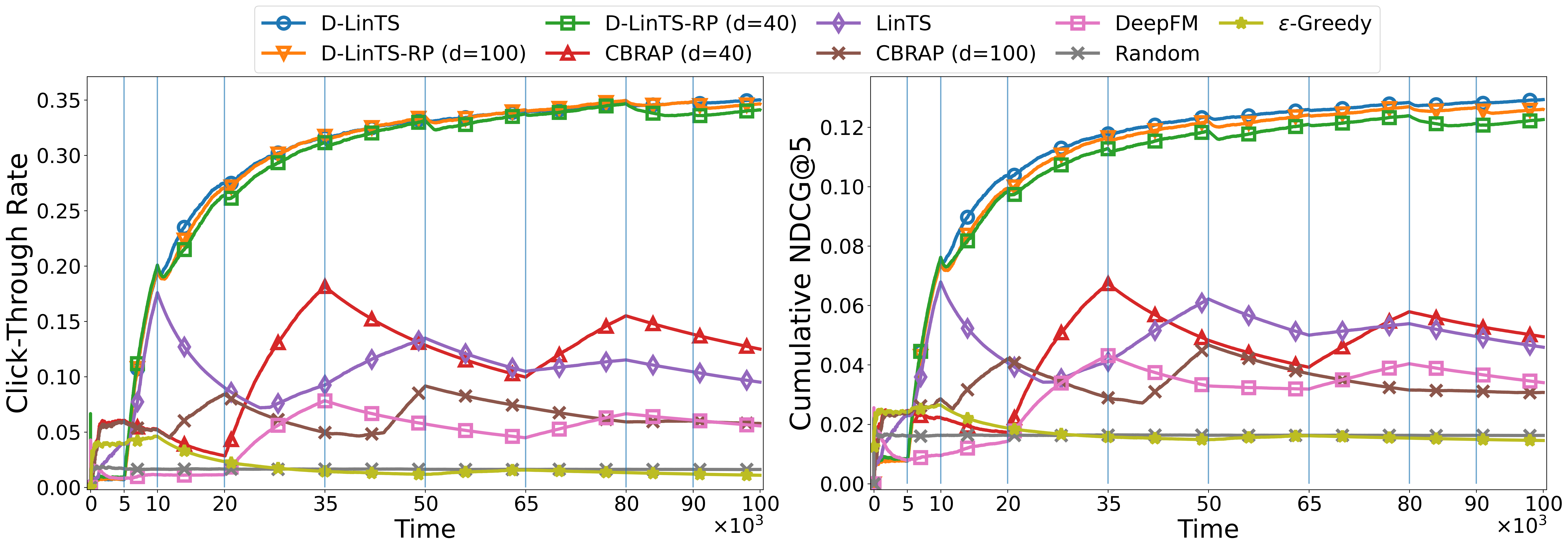} 
         \caption{Results for Amazon Books dataset.}
         \label{subfig:amazon_ctr_ndcg}
     \end{subfigure}
     \\
    \caption{The CTR and NDCG@5 of different policies over time. Vertical lines show the change points.}
    \label{fig:all_datasets_ctr_ndcg}
\end{figure*}
%-------------------------
% \subsubsection{CTR and NDCG Comparison} 
\textbf{CTR and NDCG Comparison:} 
We depict the average CTR and the average cumulative NDCG@5 of different policies for the MovieLens 10M, Jester, and Amazon Books datasets in \textbf{Fig. \ref{subfig:movielens_ctr_ndcg}, \ref{subfig:jester_ctr_ndcg}, \textup{and} \ref{subfig:amazon_ctr_ndcg}}, respectively. The results show the importance of adapting to a non-stationary environment; algorithms that adapt to changes in the reward-generating processes, i.e., D-LinTS and D-LinTS-RP, achieve higher CTR and NDCG than policies that do not recommend items adaptively. As we see, D-LinTS and D-LinTS-RP achieve comparable results; however, D-LinTS-RP uses the reduced context vectors, performing more efficiently in terms of runtime compared to D-LinTS. Note that, DeepFM performs poorly in terms of achieved reward in our experiments. This is due to the fact that DeepFM policy chooses arms based on estimated rewards, effectively doing only exploitation and no exploration. In addition, re-training the model from scratch with newly collected data does not help the DeepFM to improve its performance compared to other online benchmark methods.
% , despite using a state-of-the-art model for CTR prediction and the full context vector.
% It is outperformed even by algorithms designed for stationary environments, such as LinUCB or LinTS.
% The average cumulative NDCG@5 for the Movielens 10M, Jester, and Amazon Books datasets are presented in \textbf{Fig. \ref{subfig:movielens_ndcg}, \ref{subfig:jester_ndcg}, \textup{and} Fig. \ref{subfig:amazon_ndcg}}, respectively.

% best arm but also to output the scores that the policy used internally to inform its choice. 
The figures depicting the cumulative NDCG@5 provide additional insight into the relative performance of the algorithms. The NDCG metric assesses the ability of the evaluated policies to rank the items. NDCG@5 is defined as $\frac{DCG@5}{IDCG@5}$ where $DCG@5 = \sum_{i=1}^{5}\frac{rel_i}{\log_2{i+1}}$ and $IDCG@5$ is Ideal DCG@5, or the highest achievable DCG@5 for the given collection of items. Moreover, $rel_i$ denotes the true relevance of the item at position $i$, and the items' positions are obtained by sorting the items according to the predicted relevance. In words, NDCG quantifies how well an algorithm can predict the relevance of each item. In our experiments, instantaneous arm rewards are used as true relevance scores for NDCG computation. We adapt the bandit policies to output the best arm and a score for each arm at each decision-making round. The scores are defined based on the decision-making strategy used by an algorithm and lead to a natural ranking of arms induced by that policy. Therefore, we use these scores as predicted relevance. For D-LinTS, LinTS, and D-LinTS-RP, the score of each arm $a \in \mathcal{A}$ is the estimated expected reward $\tilde{r}_{a,t} = \boldsymbol{\tilde{\psi}}_{t}^{\top} \boldsymbol{z}_{a,t}$. For CBRAP, the upper confidence bound is used as the predicted relevance. For DeepFM and $\varepsilon$-Greedy, estimated rewards are used as the predicted relevance. For the random policy, the predicted relevance values are sampled uniformly at random from the $[0,1]$ interval. This approach allows us to assess the ability of the evaluated policies to not only identify the best item to recommend to a user but also to rank a set of relevant items in a manner that accurately reflects their relevance. As can be seen from the figures, D-LinTS-RP, with $50\%$ and $80\%$ reduction in context dimension, exhibits a performance close to that of D-LinTS and outperforms all the other benchmark policies on this ranking metric. As expected, the ability to rank the items gradually diminishes as we consider smaller values for the reduced context dimension.

% in \textbf{Fig. \ref{subfig:movielens_ctr_ndcg_bcmabrp_dlints}, \ref{subfig:jester_ctr_ndcg_bcmabrp_dlints}, \textup{and} \ref{subfig:amazon_ctr_ndcg_bcmabrp_dlints}}, respectively. 
% \subsubsection{Effect of Reduced Dimension d}
\textbf{Effect of Reduced Dimension $d$:} To further elaborate on the impact of the reduced dimension $d$, we compare the CTR and NDCG performance of the D-LinTS-RP with various reduced dimensions to those of the D-LinTS algorithm. In \textbf{Fig. \ref{fig:movielens_ctr_ndcg_bcmabrp_dlints}, \ref{fig:jester_ctr_ndcg_bcmabrp_dlints}}, and \textbf{\ref{fig:amazon_ctr_ndcg_bcmabrp_dlints}}, we depict the results for MovieLens 10M, Jester, and Amazon Books datasets, respectively. When we increase the reduced dimension $d$, the performance of D-LinTS-RP approaches that of D-LinTS in terms of CTR and NDCG, with D-LinTS-RP matching D-LinTS eventually by using $d$ equal to the original context dimension. 
% In addition, the runtime of D-LinTS-RP increases and approaches that of D-LinTS as the reduced dimension increases.
% % the original context vector with full dimension. 
% % When using $d$ equal to the original context dimension, the performance of D-LinTS-RP approaches that of D-LinTS in terms of both runtime and reward, as expected.
Therefore, for large-scale recommender systems with changing users' interests, D-LinTS-RP is a better choice than D-LinTS, provided that we use a suitable reduced dimension as the input to the algorithm.
% \textbf{Fig. \ref{fig:movielens_ctr_ndcg_bcmabrp_dlints}} depicts the NDCG metric of D-LinTS-RP for the MovieLens 10M dataset for different values of reduced context dimension $d$. 
% It is evident that the ability to rank the items gradually diminishes as we consider smaller values for the reduced context dimension.

% \textbf{Fig. \ref{fig:jester_ctr_ndcg_bcmabrp_dlints}}, and \textbf{\ref{fig:amazon_ctr_ndcg_bcmabrp_dlints}} compare the performance of the D-LinTS-RP with various reduced dimensions to the D-LinTS algorithm for Jester and Amazon Books datasets, respectively. Similar to the results corresponding to the MovieLens dataset, when we increase the reduced dimension $d$, the regret of D-LinTS-RP decreases, and the performance of D-LinTS-RP approaches that of D-LinTS in terms of both runtime and reward. When using $d$ equal to the original context dimension, the performance of D-LinTS-RP approaches that of D-LinTS in terms of both runtime and reward, as expected.
% , with D-LinTS-RP almost matching D-LinTS eventually by using $d$ equal to the original context dimension.
%------------------------->
\begin{figure*}[t!]
    \centering
     \begin{subfigure}[t]{1\textwidth}
         \centering
         \includegraphics[width=1\textwidth]{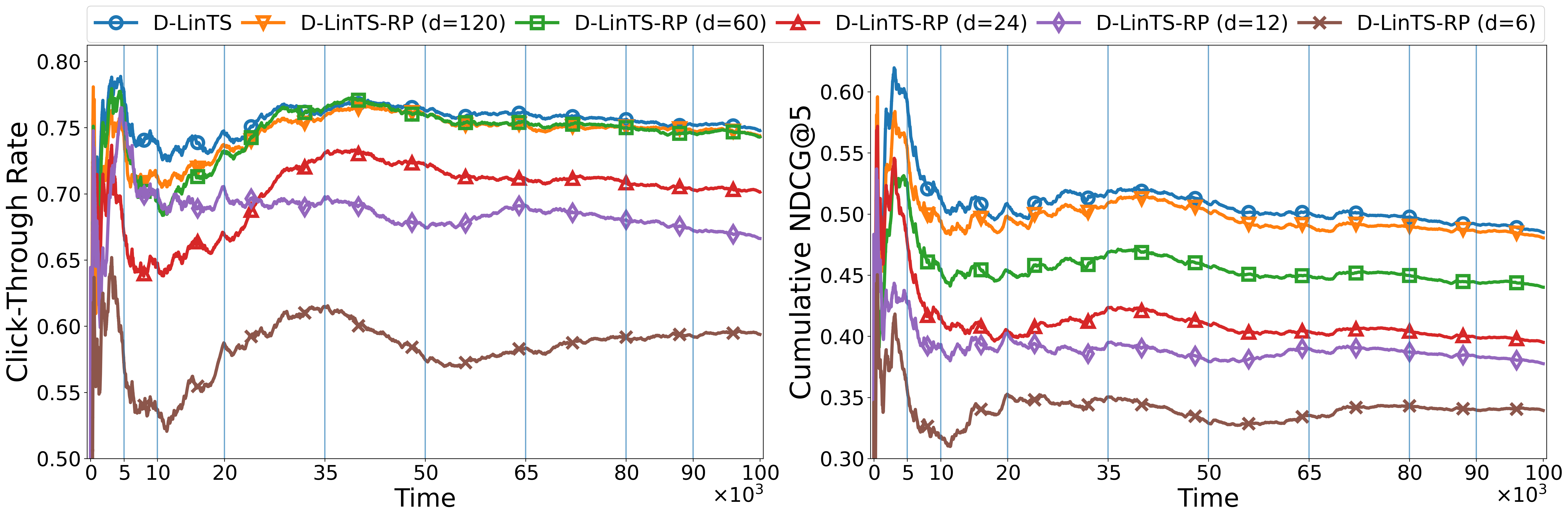} 
         \caption{Results for MovieLens 10M dataset.}
         \vspace{3mm}
\label{fig:movielens_ctr_ndcg_bcmabrp_dlints}
     \end{subfigure}
     \begin{subfigure}[t]{1\textwidth}
         \centering
         \includegraphics[width=1\textwidth]{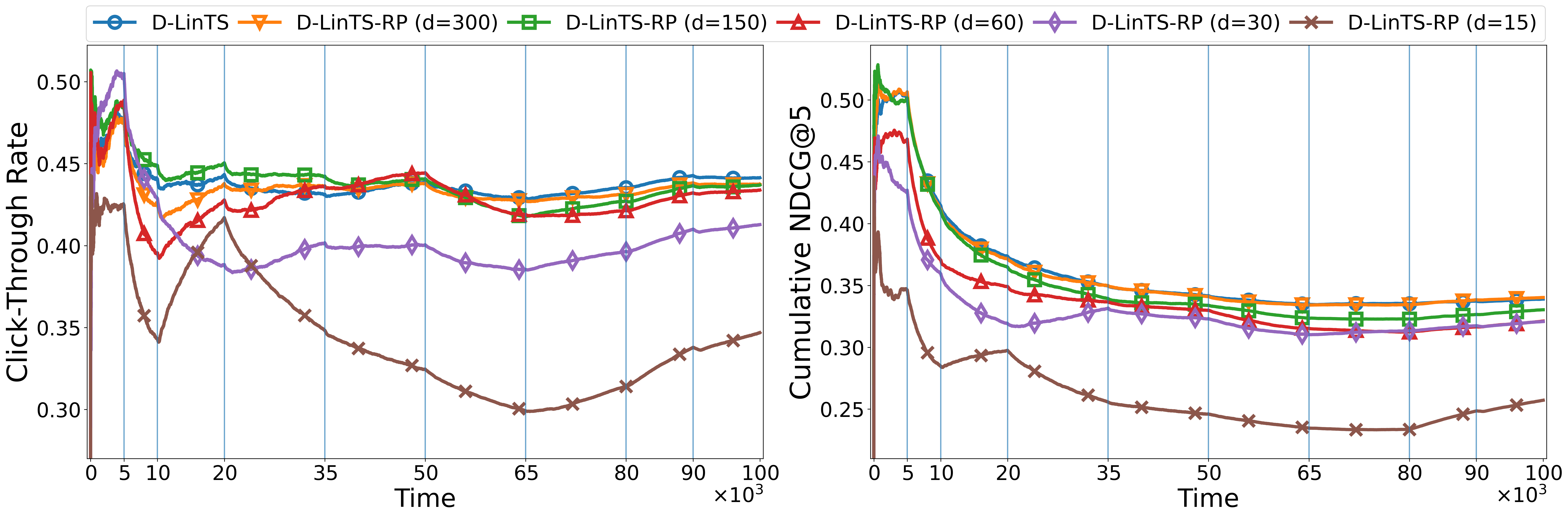} 
         \caption{Results for Jester dataset.}
         \vspace{3mm}
         \label{fig:jester_ctr_ndcg_bcmabrp_dlints}
     \end{subfigure}
     \begin{subfigure}[t]{1\textwidth}
         \centering
         \includegraphics[width=1\textwidth]{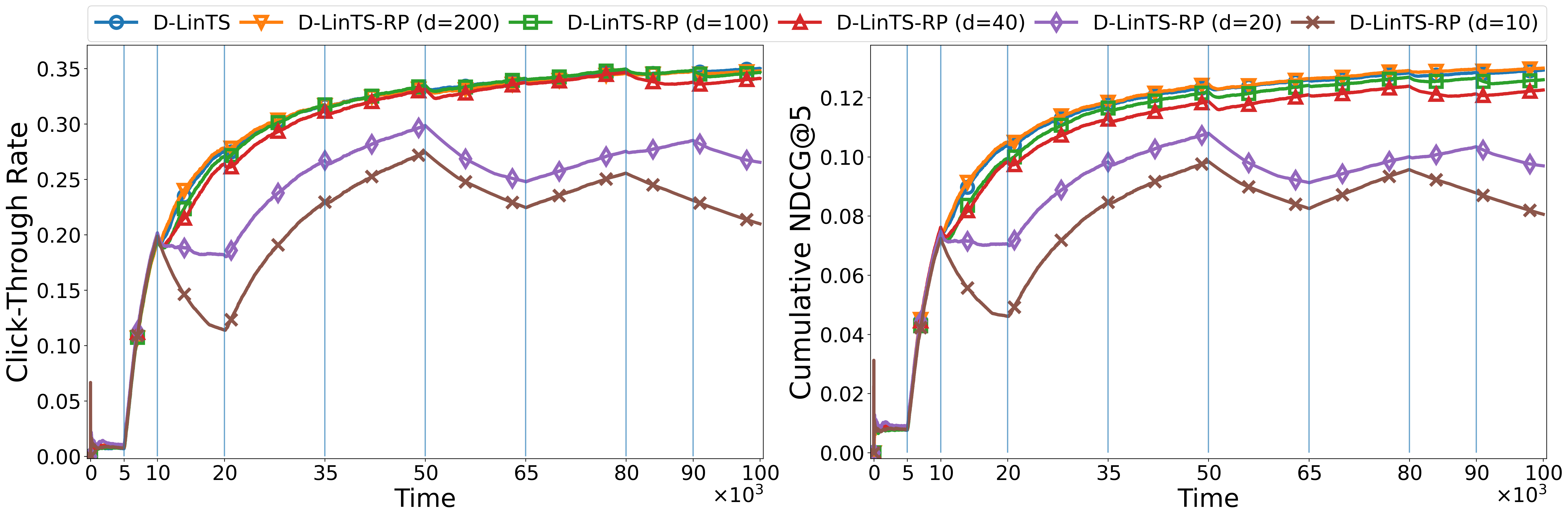} 
         \caption{Results for Amazon Books dataset.}
         \label{fig:amazon_ctr_ndcg_bcmabrp_dlints}
     \end{subfigure}
     \\
    \caption{The CTR and NDCG@5 of D-LinTS and D-LinTS-RP policies over time. Vertical lines show the change points.}
\label{fig:all_datasets_ctr_ndcg_bcmabrp_dlints}
\end{figure*}
%-------------------------

As evident from the theoretical and numerical results, the choice of the reduced dimension $d$ impacts the performance of the D-LinTS-RP algorithm. When we increase the reduced dimension $d$, the regret of D-LinTS-RP decreases, and the performance of D-LinTS-RP approaches that of D-LinTS in terms of both runtime and reward. Although larger values of $d$ expand the regret bound, this does not necessarily mean that the achieved cumulative reward will be different in practice. That is the case for our experiment on the Jester dataset, where the cumulative reward is not affected much as we decrease the value of reduced dimension $d$.
%-------------------------> Figure
% \begin{figure*}[t!]
% \centering
% \includegraphics[width=1\textwidth]{Images/bcmabrp_dlints/movielens_results_both_bcmabrp_dlints.png}
% \caption{The CTR and NDCG@5 of D-LinTS and D-LinTS-RP policies over time (MovieLens 10M dataset). Vertical lines show the change points.}
% \label{fig:movielens_ctr_ndcg_bcmabrp_dlints}
% \end{figure*}
%-------------------------

% \subsubsection{Trade-off between Computational Complexity and Regret Bound}
\textbf{Trade-off between Computational
Complexity and Regret Bound:}
The discussion above suggests that the reduced dimension $d$ makes a trade-off between the computational complexity of our algorithm and its regret performance. As expected, large $d$ increases the runtime, while choosing a small $d$ might yield information loss, thereby harming the performance w.r.t. the accumulated rewards. However, the results show that there are reduced dimensions using which D-LinTS-RP's runtime drops significantly while the algorithm continues to enjoy a high cumulative reward (See \textbf{Table \ref{Table:results-full}} in Appendix \ref{app:AddExps} for the full results). For example, for the MovieLens dataset, reducing the context dimension to $50\%$ results in just a $0.1\%$ drop in achieved reward while reducing the runtime by $29.7\%$. For the Jester dataset, reducing the context dimension to $20\%$ results in only a $0.8\%$ drop in accumulated reward, while the runtime of the algorithm decreases significantly by $90\%$. Finally, for the Amazon Books dataset, reducing the context dimension to $20\%$ results in a $2\%$ reduction in accumulated reward but leads to a $77\%$ decrease in the runtime.
% in the supplementary material
% For the Jester dataset, this effect is even stronger;
% For example, for MovieLens dataset, as we increase the reduced dimension $d$, the runtime increases while the regret decreases. 
% Our experimental evaluation attests to the existence of a trade-off between algorithm runtime and the achieved reward. As can be seen in \textbf{Table \ref{Table:results}}, when we increase the reduced dimension $d$, both the achieved cumulative reward and the runtime of D-LinTS-RP increase. 
% CBRAP policy, while achieving lower cumulative rewards than D-LinTS-RP, exhibits the same trend. 
% Our experiments show that D-LinTS-RP is able to find a balance between those two metrics: 

%-------------------------> Section Conclusion
\section{Conclusion}
\label{sec:conclusion}
We developed a decision-making policy, namely D-LinTS-RP, for the linear CMAB problem that is implementable in recommender systems. D-LinTS-RP is specifically suitable for scenarios with a large set of items, high-dimensional side information, and non-stationary environments. The policy utilizes a weighted least-squares estimator and takes advantage of random projection and exponentially increasing weights to reduce the dimension of the context vectors and the influence of past observations, respectively. We theoretically analyzed D-LinTS-RP and proved an upper regret bound that depends on the reduced dimension of the context vector. For numerical evaluation, we apply D-LinTS-RP on real-world datasets for content recommendation. The results demonstrate its effectiveness in making a trade-off between computational complexity and regret performance in non-stationary environments. Besides developing online content recommender systems, our work fits several real-world application domains, such as edge computing, medical applications, and stock trading.
\section{Appendix}
\label{sec:App}
%
% \appendices
%
%-------------------------------> Notations
%
\subsection{Notations}
\label{sec:notations}
%
% To avoid any ambiguity, in this document, we continue the (equation-) numbering of the main text.
At each time $t$, we define 
$\boldsymbol{U}_{t} = \boldsymbol{Z}_{t} \tilde{\boldsymbol{Z}}_{t}^{-1} \boldsymbol{Z}_{t}$ and
% %
% \begin{align}
%     \boldsymbol{U}_{t} = \boldsymbol{Z}_{t} \tilde{\boldsymbol{Z}}_{t}^{-1} \boldsymbol{Z}_{t},
% \end{align}
% %
% and
%
\begin{align}
    s_{a,t} = \sqrt{\boldsymbol{z}_{a,t}^{\top} \boldsymbol{U}_{t}^{-1} \boldsymbol{z}_{a,t}}, \hspace{10mm} \forall a \in \mathcal{A}.
\end{align}
In addition, at each time $t$, we define
\begin{align}
    \boldsymbol{\bar{\psi}}_{t} = \boldsymbol{Z}_{t}^{-1} (\sum_{\tau=1}^{t-1} \gamma^{-\tau} \boldsymbol{z}_{a_{\tau},\tau} \boldsymbol{z}_{a_{\tau},\tau}^{\top} \boldsymbol{\psi}_{\tau}^{\ast} + \lambda \gamma^{-(t-1)} \boldsymbol{\psi}_{t}^{\ast}).
\end{align}
We define the Event  $E_{\boldsymbol{\hat{\psi}}}$ as
\begin{align}
    E_{\boldsymbol{\hat{\psi}}} = \{\forall a \in \mathcal{A}, \forall t \in [T],  |\boldsymbol{z}_{a,t}^{\top} \boldsymbol{\hat{\psi}}_{t} - \boldsymbol{z}_{a,t}^{\top} \boldsymbol{\bar{\psi}}_{t}| \leq \alpha s_{a,t}\},
\end{align}
where $\alpha = R \sqrt{2\log{T} + d\log{\left(1 + \frac{L_{\boldsymbol{z}}^{2} (1-\gamma^{2T})}{\lambda d (1-\gamma^{2})}\right)} }+ \lambda^{\frac{1}{2}} L_{\boldsymbol{\psi}} + \varepsilon \gamma \sqrt{\frac{ (1-(1/\gamma)^{T})}{\lambda (1-(1/\gamma))}}$. Moreover, we define the Event $E_{\boldsymbol{\tilde{\psi}}}(t)$ as
\begin{align}
    E_{\boldsymbol{\tilde{\psi}}}(t) = \{\forall a \in \mathcal{A}, |\boldsymbol{z}_{a,t}^{\top} \boldsymbol{\tilde{\psi}}_{t} - \boldsymbol{z}_{a,t}^{\top} \boldsymbol{\hat{\psi}}_{t}| \leq \beta s_{a,t}\},
\end{align}
where $\beta = \xi \sqrt{2 \log{(\frac{A T}{2})}}$.
In addition, we define the following Event $E_{\ast}(t)$.
\begin{align}
    E_{\ast}(t) = \{\boldsymbol{z}_{a_{t}^{\ast},t}^{\top} \boldsymbol{\tilde{\psi}}_{t} - \boldsymbol{z}_{a_{t}^{\ast},t}^{\top} \boldsymbol{\hat{\psi}}_{t} > \alpha s_{a_{t}^{\ast},t}\}.
\end{align}
Further, for any $\varepsilon \in (0,1)$, we define the Event $E_{\boldsymbol{\psi}^{\ast}}(t)$ as
\begin{align} \nonumber
    E_{\boldsymbol{\psi}^{\ast}}(t) =
    \{\forall a \in \mathcal{A}, \boldsymbol{z}_{a,t}^{\top} \boldsymbol{\psi}_{t}^{\ast}  > \boldsymbol{x}_{a,t}^{\top} \boldsymbol{\theta}_{t}^{\ast} d \kappa^{2} &- \varepsilon d \kappa^{2} \norm{\boldsymbol{x}_{a,t}}_{2} \norm{\boldsymbol{\theta}_{t}^{\ast}}_{2} \\ 
    &~\&~ \boldsymbol{z}_{a,t}^{\top} \boldsymbol{\psi}_{t}^{\ast}  < \boldsymbol{x}_{a,t}^{\top} \boldsymbol{\theta}_{t}^{\ast} d \kappa^{2} + \varepsilon d \kappa^{2} \norm{\boldsymbol{x}_{a,t}}_{2} \norm{\boldsymbol{\theta}_{t}^{\ast}}_{2}\}.
\end{align}
Let $\norm{\boldsymbol{z}_{a,t}}_{2} \leq L_{\boldsymbol{z}}$, $\forall a \in \mathcal{A}$, and $\norm{\boldsymbol{\psi}_{t}^{\ast}}_{2} \leq L_{\boldsymbol{\psi}}$, for some constants $L_{\boldsymbol{z}}$, $L_{\boldsymbol{\psi}} \geq 1$. Define $\Delta_{a}(t) = \boldsymbol{z}_{a_{t}^{\ast},t}^{\top} \boldsymbol{\bar{\psi}}_{t} - \boldsymbol{z}_{a,t}^{\top} \boldsymbol{\bar{\psi}}_{t}$, $\forall a \in \mathcal{A}$. At each time $t$, we divide the arms into the following two sets: (i) set of saturated arms, and (ii) set of unsaturated arms. 
% The set of saturated arms includes any arm for which the estimated reward's standard deviation is smaller than that of the optimal arm $a_{t}^{\ast}$. Any arm whose estimated reward's standard deviation is greater than that of the optimal arm belongs to the set of unsaturated arms.
The set of unsaturated arms includes undersampled arms that are worse than $a_{t}^{\ast}$ given $\bar{\boldsymbol{\psi}}_{t}$ at time $t$, whereas the set of saturated arms includes sufficiently sampled arms that are worse than $a_{t}^{\ast}$ given $\bar{\boldsymbol{\psi}}_{t}$ at time $t$.
Formally, at each time $t$, the set of arms $\mathcal{A}$ is divided into two following sets.
\begin{itemize}
    \item Saturated arms:
    \begin{align}
        C(t) = \{ a \in \mathcal{A} ~|~ \Delta_{a}(t) \geq \gamma s_{a,t} ~\&~ \Delta_{a}(t) \geq 0\},
    \end{align}
    \item Unsaturated arms:
    \begin{align}
        \overline{C}(t) = \{ a \in \mathcal{A} ~|~ \Delta_{a}(t) \leq \gamma s_{a,t} ~\&~ \Delta_{a}(t) \geq 0 \},
    \end{align}
\end{itemize}
where $\gamma = \alpha + \beta$.

Following the scale-invariance property of the weighted least-square estimator, we can equivalently obtain $\hat{\boldsymbol{\psi}}_{t}$ as \cite{Russac19:WLB}
\begin{align}
\label{eq:optimization-standard}
    \hat{\boldsymbol{\psi}}_{t} = \argminA_{\boldsymbol{\psi} \in \mathbbm{R}^{d}} \left( \sum_{\tau=1}^{t} \gamma^{-\tau} \left(r_{a_{t},t} - \boldsymbol{\psi}^{\top} \boldsymbol{z}_{a_{\tau},\tau} \right)^{2} + \gamma^{-t} \frac{\lambda}{2} \norm{\boldsymbol{\psi}}_{2}^{2} \right).
\end{align}
%
% %
% \begin{align}
%     \hat{\boldsymbol{\psi}}_{t} = \argminA_{\boldsymbol{\psi} \in \mathbbm{R}^{d}} \left( \sum_{\tau=1}^{t} \gamma^{t-\tau} \left(r_{a_{t},t} - \boldsymbol{\psi}^{\top} \boldsymbol{z}_{a_{\tau},\tau} \right)^{2} + \frac{\lambda}{2} \norm{\boldsymbol{\psi}}_{2}^{2} \right).
% \end{align}
% %
% Note that, while we used the estimator $(3)$ in Algorithm $1$
Note that, while we used the estimator (\ref{eq:optimization}) in Algorithm \ref{Alg:D-LinTS-RP}, we use the estimator (\ref{eq:optimization-standard}) for its regret analysis in Section \ref{app:TheoremOneProof} below, as we need to apply the concentration results in Lemma \ref{lem:2}. %(See Section \ref{app:mainresults} below)

Let $\mathcal{F}_{t} = \sigma(a_{1}, \dots, a_{t}, r_{a_{1},1}, r_{a_{t},t})$ be the $\sigma$-algebra generated by the selected arms and their rewards by the end of time $t$. We denote by $\mathbbm{P}_{t}[\cdot]$ and $\mathbbm{E}_{t}[\cdot]$ the conditional probability $\mathbbm{P}[\cdot | \mathcal{F}_{t-1}]$ and the conditional expectation $\mathbbm{E}[\cdot | \mathcal{F}_{t-1}]$, respectively. %, given the past at the beginning of time $t$.
%-------------------------------> Auxiliary
\subsection{Auxiliary Results}
\label{app:aux}
%
%-------------------------------> Lemma 1
\begin{lemma}{\textup{(\cite{abramowitz48:HOM})}} \\
\label{lem:anticoncentration}
For a Gaussian distributed random variable Y with mean $\mu$ and variance $\sigma^{2}$, for any $y > 0$, the following holds.
\begin{equation}
    \frac{1}{4 \sqrt{\pi}} \exp(\frac{-7 y^{2}}{2}) \leq
    \mathbbm{P}(|Y - \mu| > y \sigma)
    \leq \frac{1}{2} \exp(\frac{-y^{2}}{2}).
\end{equation}
\end{lemma}
%-------------------------------> Main Results
\subsection{Main Results}
\label{app:mainresults}
Before we proceed to the proof of Theorem \ref{thm:regbound}, we prove Lemmas \ref{lem:2}-\ref{lem:4}. %Theorem 1 \ref{thm:regbound}
%-------------------------------> Lemma 2
\begin{lemma}
\label{lem:2}
% \ref{eq:reward} (1)
% \ref{eq:ArmSelectionStrategy}
% \ref{Alg:D-LinTS-RP}
% Consider the linear model for reward variables defined in (1) and the arm selection strategy $a_{t} = \argmaxA_{a \in \mathcal{A}} \boldsymbol{\tilde{\psi}}_{t}^{\top} \boldsymbol{z}_{a,t}$ in Algorithm 1
Consider the linear model for reward variables defined in (\ref{eq:reward}) and the arm selection strategy $a_{t} = \argmaxA_{a \in \mathcal{A}} \boldsymbol{\tilde{\psi}}_{t}^{\top} \boldsymbol{z}_{a,t}$ in Algorithm \ref{Alg:D-LinTS-RP}. At each time $t$, for any $\lambda > 0$ and $\delta, \varepsilon \in (0,1)$, the following holds.
\begin{align}
    \mathbbm{P}(E_{\boldsymbol{\hat{\psi}}}) \geq (1 - \frac{1}{T})(1 - 2\exp(- \frac{d \varepsilon^{2}}{8})).
\end{align}
%\mathcal{H}_{t-1}^{\prime}
Moreover, given history $\mathcal{H}_{t-1}^{\prime}$, the following holds.
\begin{align}
\label{eq:lem2-hyp2}
    \mathbbm{P}(E_{\boldsymbol{\tilde{\psi}}}(t)) \geq 1 - \frac{1}{T}.
\end{align}
\end{lemma}
%-------------------------------
\begin{proof}
% (\ref{eq:eqforZ}), (\ref{eq:eqforb}), and (\ref{eq:eqforpsi})
% Based on the definitions in (6), (8), and (9), 
We have
%
% \newpage
\begin{align} \nonumber
\label{eq:lem1-1}
    (\boldsymbol{\hat{\psi}}_{t} - \boldsymbol{\bar{\psi}})
    &= \boldsymbol{Z}_{t}^{-1} {\Bigg(} \sum_{\tau = 1}^{t-1}  \gamma^{-\tau} r_{a_{\tau},\tau} \boldsymbol{z}_{a_{\tau},\tau} 
    - \sum_{\tau = 1}^{t-1} \gamma^{-\tau} \boldsymbol{z}_{a_{\tau},\tau} \boldsymbol{z}_{a_{\tau},\tau}^{\top} \boldsymbol{\psi}_{\tau}^{\ast} - \lambda \gamma^{-(t-1)} \boldsymbol{\psi}_{t}^{\ast} {\Bigg)} \\
    % &= \boldsymbol{Z}_{t}^{-1} {\Bigg(} \sum_{\tau = 1}^{t-1} \gamma^{-\tau} \boldsymbol{z}_{a_{\tau},\tau} (r_{a_{\tau},\tau} - \boldsymbol{x}_{a_{\tau},\tau}^{\top} \theta_{\tau}^{\ast} )
    % \\ \nonumber
    % &+ \sum_{\tau = 1}^{t-1} \gamma^{-\tau} \boldsymbol{z}_{a_{\tau},\tau} ( \boldsymbol{x}_{a_{\tau},\tau}^{\top} \theta_{\tau}^{\ast} -   \boldsymbol{z}_{a_{\tau},\tau}^{\top} \boldsymbol{\psi}_{\tau}^{\ast}) - 
    % \lambda \gamma^{-(t-1)} \boldsymbol{\psi}_{t}^{\ast} {\Bigg)} \\ \nonumber
    &= \boldsymbol{Z}_{t}^{-1} {\Bigg(} \sum_{\tau = 1}^{t-1} \gamma^{-\tau} \boldsymbol{z}_{a_{\tau},\tau} \eta_{t} 
    + \sum_{\tau = 1}^{t-1} \gamma^{-\tau} \boldsymbol{z}_{a_{\tau},\tau} ( \boldsymbol{x}_{a_{\tau},\tau}^{\top} \theta_{\tau}^{\ast} -   \boldsymbol{z}_{a_{\tau},\tau}^{\top} \boldsymbol{\psi}_{\tau}^{\ast}) 
    - \lambda \gamma^{-(t-1)} \boldsymbol{\psi}_{t}^{\ast} {\Bigg)}.
\end{align}
% \\
%     &\hspace{80mm}
%
Moreover, we observe that \cite{Russac19:WLB}
%since $\lambda \leq \mu_{\min}(\boldsymbol{Z}_{t})$, 
\begin{align} 
\label{eq:secondpart}
     \lambda \gamma^{-(t-1)} \norm{\boldsymbol{\psi}_{t}^{\ast}}_{\tilde{\boldsymbol{Z}}_{t}^{-1}} \leq \lambda^{\frac{1}{2}} L_{\boldsymbol{\psi}}.
\end{align}
In addition, for any arm $a \in \mathcal{A}$, we have \cite{Kaban15:IBO}
\begin{align}
\label{eq:orgvsred1}
    &\mathbbm{P}(\boldsymbol{z}_{a,t}^{\top} \boldsymbol{\psi}_{t}^{\ast}  < \boldsymbol{x}_{a,t}^{\top} \boldsymbol{\theta}_{t}^{\ast} d \kappa^{2} - \varepsilon d \kappa^{2} \norm{\boldsymbol{x}_{a,t}}_{2} \norm{\boldsymbol{\theta}_{t}^{\ast}}_{2}) < \exp(- \frac{d \varepsilon^{2}}{8}), \\ 
\label{eq:orgvsred2}
    &\mathbbm{P}(\boldsymbol{z}_{a,t}^{\top} \boldsymbol{\psi}_{t}^{\ast}  > \boldsymbol{x}_{a,t}^{\top} \boldsymbol{\theta}_{t}^{\ast} d \kappa^{2} + \varepsilon d \kappa^{2} \norm{\boldsymbol{x}_{a,t}}_{2} \norm{\boldsymbol{\theta}_{t}^{\ast}}_{2}) < \exp(- \frac{d \varepsilon^{2}}{8}).
\end{align}
Based on our assumption on the construction of the random matrix $\boldsymbol{P}$, it holds $d \kappa^{2} = 1$. Moreover, we have $\norm{\boldsymbol{x}_{a,t}}_{2} \leq 1$ and $\norm{\boldsymbol{\theta}_{t}^{\ast}}_{2} \leq 1$. Therefore, based on (\ref{eq:orgvsred1}) and (\ref{eq:orgvsred2}), with probability at least $1 - 2\exp(- \frac{d \varepsilon^{2}}{8})$, the following holds.
\begin{align} \nonumber
\label{eq:thirdpart}
    \norm{\sum_{\tau = 1}^{t-1} \gamma^{-\tau} \boldsymbol{z}_{a_{\tau},\tau} ( \boldsymbol{x}_{a_{\tau},\tau}^{\top} \theta_{\tau}^{\ast} - \boldsymbol{z}_{a_{\tau},\tau}^{\top} \boldsymbol{\psi}_{\tau}^{\ast} )}_{\tilde{\boldsymbol{Z}}_{t}^{-1}} 
    &\leq \frac{\varepsilon}{\sqrt{\lambda} \gamma^{-(t-1)}} \norm{\sum_{\tau = 1}^{t-1} \gamma^{-\tau} }_{2} \\
    &\leq \varepsilon \gamma \sqrt{\frac{ (1-(1/\gamma)^{T})}{\lambda (1-(1/\gamma))}}. %\gamma^{(t-1)}
    % \frac{\varepsilon (1-\gamma^{t})}{\gamma (1-\gamma)}. \textup{OR}
\end{align}
Summarizing the above results, with probability at least $(1 - \frac{1}{T})(1 - 2 \exp(- \frac{d \varepsilon^{2}}{8}))$, the following holds.
\begin{align} \nonumber
    | \boldsymbol{z}_{a,t}^{\top} (\boldsymbol{\hat{\psi}}_{t} - \bar{\boldsymbol{\psi}}_{t})|
    &\stackrel{\text{$(a)$}}{\leq} {\Bigg (} \norm{\sum_{\tau = 1}^{t-1} \gamma^{-\tau} \boldsymbol{z}_{a_{\tau},\tau} \eta_{t} }_{\tilde{\boldsymbol{Z}}_{t}^{-1}}
    + \lambda \gamma^{-(t-1)} \norm{ \boldsymbol{\psi}_{t}^{\ast}}_{\tilde{\boldsymbol{Z}}_{t}^{-1}} \\ \nonumber
    &\hspace{25mm}+
    \norm{\sum_{\tau = 1}^{t-1} \gamma^{-\tau} 
    \boldsymbol{z}_{a_{\tau},\tau} ( \boldsymbol{x}_{a_{\tau},\tau}^{\top} \theta_{\tau}^{\ast}  - \boldsymbol{z}_{a_{\tau},\tau}^{\top} \boldsymbol{\psi}_{t}^{\ast}  )}_{\tilde{\boldsymbol{Z}}_{t}^{-1}} {\Bigg )} \norm{\boldsymbol{z}_{a,t}}_{\boldsymbol{U}_{t}^{-1}} \\ \nonumber
    &\stackrel{\text{$(b)$}}{\leq} {\Bigg(} R \sqrt{2\log{T} + d\log{\left(1 + \frac{L_{\boldsymbol{z}}^{2} (1-\gamma^{2T})}{\lambda d (1-\gamma^{2})}\right)} } \\ \nonumber
    &\hspace{25mm}+ \lambda^{\frac{1}{2}} L_{\boldsymbol{\psi}} + \varepsilon \gamma  \sqrt{\frac{ (1-(1/\gamma)^{T})}{\lambda (1-(1/\gamma))}} {\Bigg)} \norm{\boldsymbol{z}_{a,t}}_{\boldsymbol{U}_{t}^{-1}} \\
    &= \alpha s_{a,t},
\end{align}
where $(a)$ follows from (\ref{eq:lem1-1}) and $(b)$ follows from (\ref{eq:secondpart}), (\ref{eq:thirdpart}), and slight modification of Proposition 3 in \cite{Russac19:WLB} .
% Theorem \ref{thm:selfnormalizedbound}.

% For the second event $\mathbbm{P}(E_{\boldsymbol{\tilde{\psi}}}(t) | \mathcal{F}_{t-1}^{\prime})$, the result follows directly from \cite{Baekjin20:REF}. 
For the second event, given history $\mathcal{H}_{t-1}^{\prime}$ and using the linear invariant property of Gaussian distributions, we have $\boldsymbol{z}_{a,t}^{\top} \boldsymbol{\tilde{\psi}}_{t}  = \boldsymbol{z}_{a,t}^{\top} \boldsymbol{\hat{\psi}}_{t} + \boldsymbol{z}_{a,t}^{\top} \boldsymbol{Z}_{t}^{-1} \boldsymbol{\tilde{Z}}_{t}^{1/2} \boldsymbol{W}$ is equivalent to $\boldsymbol{z}_{a,t}^{\top} \boldsymbol{\tilde{\psi}}_{t}  = \boldsymbol{z}_{a,t}^{\top} \boldsymbol{\hat{\psi}}_{t} + \boldsymbol{Z}_{t, \boldsymbol{z}_{a,t}}  s_{a,t}$, where $\boldsymbol{\boldsymbol{Z}}_{t, \boldsymbol{z}_{a,t}} \sim \mathcal{N}(0, \xi^{2})$. Then, 
\begin{align} \nonumber
    \mathbbm{P}(E_{\boldsymbol{\tilde{\psi}}}(t)) 
    = \mathbbm{P}(|\boldsymbol{z}_{a,t}^{\top} \boldsymbol{\tilde{\psi}}_{t}- \boldsymbol{z}_{a,t}^{\top} \boldsymbol{\hat{\psi}}_{t} | \leq \beta s_{a,t},~\forall a \in \mathcal{A})
    &= \mathbbm{P}(|\boldsymbol{Z}_{t, \boldsymbol{z}_{a,t}}| s_{a,t} \leq \beta s_{a,t},~\forall a \in \mathcal{A}) \\ \nonumber
    &= \mathbbm{P}(|\boldsymbol{Z}_{t, \boldsymbol{z}_{a,t}}| \leq \beta,~\forall a \in \mathcal{A}) \\
    &\stackrel{\text{$(\ast)$}}{\geq} 1 - 1/T,
\end{align}
where $(\ast)$ follows from Lemma \ref{lem:anticoncentration}.
\end{proof}
%-------------------------------
%-------------------------------> Lemma 3
\begin{lemma}
\label{lem:3}
Let $\xi^{2} = 14 \alpha^{2}$. Given $\mathcal{H}_{t-1}^{\prime}$,
\begin{align}
\label{eq:lem3-hyp1}
   \mathbbm{P}(E_{\ast}(t)) \geq \frac{e^{-1/4}}{8 \sqrt{\pi}}.
\end{align}
\end{lemma}
%-------------------------------
\begin{proof}
The proof is inspired by \cite{Baekjin20:REF} and the difference here is that we are working in the low-dimensional space $\mathbbm{R}^{d}$ instead of the original space $\mathbbm{R}^{n}$. Similar to the proof of (\ref{eq:lem2-hyp2}) in Lemma \ref{lem:2}, we observe that $\boldsymbol{z}_{a,t}^{\top} \boldsymbol{\tilde{\psi}}_{t}  = \boldsymbol{z}_{a,t}^{\top} \boldsymbol{\hat{\psi}}_{t} + \boldsymbol{z}_{a,t}^{\top} \boldsymbol{Z}_{t}^{-1} \boldsymbol{\tilde{Z}}_{t}^{1/2} \boldsymbol{W}$ is equivalent to $\boldsymbol{z}_{a,t}^{\top} \boldsymbol{\tilde{\psi}}_{t}  = \boldsymbol{z}_{a,t}^{\top} \boldsymbol{\hat{\psi}}_{t} + \boldsymbol{Z}_{t, \boldsymbol{z}_{a,t}}  s_{a,t}$, where $\boldsymbol{\boldsymbol{Z}}_{t, \boldsymbol{z}_{a,t}} \sim \mathcal{N}(0, \xi^{2})$. Therefore,
\begin{align} 
    \mathbbm{P}(
    \boldsymbol{z}_{a_{t}^{\ast},t}^{\top} \boldsymbol{\tilde{\psi}}_{t} - \boldsymbol{z}_{a_{t}^{\ast},t}^{\top} \boldsymbol{\hat{\psi}}_{t} > \alpha s_{a_{t}^{\ast},t}) 
    = \mathbbm{P}(\boldsymbol{Z}_{t, \boldsymbol{z}_{a_{t}^{\ast},t}} > \alpha) 
    \stackrel{\text{$(\ast)$}}{\geq} \frac{e^{-1/4}}{8 \sqrt{\pi}},
\end{align}
where $(\ast)$ follows from Lemma \ref{lem:anticoncentration}.
\end{proof}
%-------------------------------
%-------------------------------> Lemma 4
\begin{lemma}
% \begin{proposition}
% \end{proposition}
\label{lem:4}
Let $\lambda \geq 1$. Assume $\alpha, \beta \geq 1$ such that $\mathbbm{P}(E_{\boldsymbol{\hat{\psi}}}) \geq 1 - p_{1}$, $\mathbbm{P}(E_{\boldsymbol{\tilde{\psi}}}(t)) \geq 1 - p_{2}$, and $\mathbbm{P}(E_{\ast}(t)) \geq p_{3}$, where $p_{1} = 2\exp(- \frac{d \varepsilon^{2}}{8}) + \frac{1}{T}(1-2\exp(- \frac{d \varepsilon^{2}}{8}))$, $p_{2} = \frac{1}{T}$, and $p_{3} = \frac{e^{-1/4}}{8 \sqrt{\pi}}$. Then, %the expected surrogate instantaneous regret is bounded as
% Let $\zeta = 2d\log(\frac{1}{\gamma}) + 2 \frac{d}{T} \log(1 + \frac{L_{\boldsymbol{z}}^{2}}{d\lambda(1-\gamma)})$. 
%
\begin{align}
   \mathbbm{E}[ \Delta_{a_{t}}(t)] 
   &\leq (\alpha + \beta) (1 + \frac{2}{p_{3} - p_{2}}) \mathbbm{E}_{t}[s_{a_{t},t}] + 2 L_{\boldsymbol{\psi}} L_{\boldsymbol{z}} p_{2}.
\end{align}
%\ceil{L_{\boldsymbol{z}}} \mathbbm{E}[\min\{1, s_{a_{t},t}\}]
% where $p = \frac{1}{4 e^{4} \sqrt{\pi}}$.
%
\end{lemma}
% $\newline$
% $\newline$
%-------------------------------
\begin{proof}
The proof is inspired by \cite{Baekjin20:REF}. The difference here is that we are working in the low-dimensional space $\mathbbm{R}^{d}$ instead of the original space $\mathbbm{R}^{n}$.

Let $\overline{a}_{t}$ denote the unsaturated arm with the smallest $s_{a,t}$. Formally,
\begin{align}
\label{eq:leastuncertain}
    \overline{a}_{t} = \argminA_{a \in \overline{C}(t)} s_{a,t}.
\end{align}
Note that $\overline{a}_{t}$ is fixed since $\overline{C}(t)$ and $s_{a,t}$, $\forall a \in \mathcal{A}$, are fixed for a given $\mathcal{H}_{t-1}^{\prime}$. Based on the definition of the optimal arm $a_{t}^{\ast}$, we know that $a_{t}^{\ast} \in \overline{C}(t)$.
% Note that, by definition of the set of saturated and unsaturated arms, we neglect any action $a \in \mathcal{A}$ with $\Delta_{a}(t) > 0$, since the regret induced by such actions are always negative and therefore, they are upper bounded by zero.
When both the Events $E_{\boldsymbol{\hat{\psi}}}$ and $E_{\boldsymbol{\tilde{\psi}}}(t)$ are true, we have
\begin{align} \nonumber
    \Delta_{a_{t}}(t) 
    = \Delta_{\overline{a}_{t}}(t) + \boldsymbol{z}_{\overline{a}_{t},t}^{\top} \boldsymbol{\bar{\psi}}_{t} -  \boldsymbol{z}_{a_{t},t}^{\top} \boldsymbol{\bar{\psi}}_{t}
    &\stackrel{\text{$(a)$}}{\leq} \Delta_{\overline{a}_{t}}(t) + ( \tilde{r}_{\overline{a}_{t},t} - \tilde{r}_{a_{t},t}) + ( \gamma s_{\overline{a}_{t},t} + \gamma s_{a_{t},t}) \\ \nonumber
    &\stackrel{\text{$(b)$}}{\leq} \Delta_{\overline{a}_{t}}(t) + ( \gamma s_{\overline{a}_{t},t} + \gamma s_{a_{t},t}) \\ \label{eq:lem5-2}
    &\stackrel{\text{$(c)$}}{\leq} 2 \gamma s_{\overline{a}_{t},t} + \gamma s_{a_{t},t},
\end{align}
where $(a)$ holds since both $E_{\boldsymbol{\hat{\psi}}}$ and $E_{\boldsymbol{\tilde{\psi}}}(t)$ are true. Moreover, $(b)$ follows from the fact that at time $t$ we have $\tilde{r}_{a_{t},t} \geq \tilde{r}_{a,t}$, $\forall a$. Finally, $(c)$ is concluded from the definition of unsaturated arms.

Let $I(E)$ represent the indicator function which is equal to $1$ if the Event $E$ happens, and is $0$ otherwise. Therefore,
\begin{align} \nonumber
    \mathbbm{E}_{t}[\Delta_{a_{t}}(t)]
    &= \mathbbm{E}_{t}[\Delta_{a_{t}}(t) I\{E_{\boldsymbol{\tilde{\psi}}}(t)\} ] + \mathbbm{E}_{t}[\Delta_{a_{t}}(t)  I\{\overline{E_{\boldsymbol{\tilde{\psi}}}(t)}\}]
    \\ \nonumber
    &\stackrel{\text{$(a)$}}{\leq} 2 \gamma s_{\overline{a}_{t},t} + \gamma \mathbbm{E}_{t}[s_{a_{t},t}] + 2 L_{\boldsymbol{\psi}} L_{\boldsymbol{z}} \mathbbm{P}_{t}( \overline{E_{\boldsymbol{\tilde{\psi}}}(t)}) \\ \nonumber
    &\stackrel{\text{$(b)$}}{\leq} 2 \gamma \frac{\mathbbm{E}_{t}[s_{a_{t},t}]}{\mathbbm{P}_{t}(a_{t} \in \overline{C}(t))} + \gamma \mathbbm{E}_{t}[s_{a_{t},t}] + 2 L_{\boldsymbol{\psi}} L_{\boldsymbol{z}} p_{2} \\ \nonumber
    &=  \gamma( 1 + \frac{2}{\mathbbm{P}_{t}(a_{t} \in \overline{C}(t))}) \mathbbm{E}_{t}[s_{a_{t},t}] + 2 L_{\boldsymbol{\psi}} L_{\boldsymbol{z}} p_{2}, \\ %\nonumber
    &\stackrel{\text{$(c)$}}{\leq} \gamma( 1 + \frac{2}{p_{3} - p_{2}}) \mathbbm{E}_{t}[s_{a_{t},t}] + 2 L_{\boldsymbol{\psi}} L_{\boldsymbol{z}} p_{2}, 
    % \\
    % &\stackrel{\text{$(d)$}}{\leq} \gamma( 1 + \frac{2}{p_{3} - p_{2}}) \mathbbm{E}[\min\{L_{\boldsymbol{z}}, s_{a_{t},t}\}] + 2 L_{\boldsymbol{\psi}} L_{\boldsymbol{z}} p_{2},
\end{align}
%and the fact that $\Delta_{a}(t) \leq 2 L_{\boldsymbol{\psi}} L_{\boldsymbol{z}}$, $\forall a$.
where $(a)$ follows from (\ref{eq:lem5-2}). Moreover, $(b)$ holds since
% Given $\mathcal{H}_{t-1}^{\prime}$ such that $E_{\boldsymbol{\hat{\psi}}}$ is true, we have
%
\begin{align} 
\label{eq:lem5-1}
    \mathbbm{E}_{t}[s_{a_{t},t} | \mathcal{F}_{t-1}^{\prime}] 
    \geq \mathbbm{E}_{t}[s_{a_{t},t} | a_{t} \in \overline{C}(t)] \mathbbm{P}_{t}(a_{t} \in \overline{C}(t))  
    \geq s_{\overline{a}_{t},t} \mathbbm{P}_{t}(a_{t} \in \overline{C}(t)),
\end{align}
where the last inequality follows from (\ref{eq:leastuncertain}).
Further, $(c)$ holds since when $E_{\boldsymbol{\hat{\psi}}}$ is true, we have
\begin{align} \nonumber
    \mathbbm{P}_{t}(a_{t} \in \overline{C}(t))
    \geq \mathbbm{P}_{t}(\exists a \in \overline{C}(t) ~\textup{s.t.}~ \tilde{r}_{a,t} > \max_{a' \in C(t)} \tilde{r}_{a',t})
    &\geq \mathbbm{P}_{t}(\tilde{r}_{a_{t}^{\ast},t} > \max_{a' \in C(t)} \tilde{r}_{a',t}) \\ \nonumber
    &\geq \mathbbm{P}_{t}(\tilde{r}_{a_{t}^{\ast},t} > \max_{a' \in C(t)} \tilde{r}_{a',t}, E_{\boldsymbol{\tilde{\psi}}}(t)) \\ \nonumber
    &\stackrel{\text{$(\ast)$}}{\geq} \mathbbm{P}_{t}(\tilde{r}_{a_{t}^{\ast},t} >  \boldsymbol{z}_{a_{t}^{\ast},t}^{\top} \boldsymbol{\bar{\psi}}_{t}, E_{\boldsymbol{\tilde{\psi}}}(t)) \\ \nonumber
    &\stackrel{\text{$(\ast\ast)$}}{\geq} \mathbbm{P}_{t}(\tilde{r}_{a_{t}^{\ast},t} > \boldsymbol{z}_{a_{t}^{\ast},t}^{\top} \boldsymbol{\bar{\psi}}_{t}) - \mathbbm{P}( \overline{E_{\boldsymbol{\tilde{\psi}}}(t)}) \\
    &\geq p_{3} - p_{2},
\end{align}
where $(\ast)$ holds since for any $a \in C(t)$,
\begin{align}
   \tilde{r}_{a,t} \leq \boldsymbol{z}_{a,t}^{\top} \bar{\boldsymbol{\psi}}_{t} + \gamma s_{a,t} \leq  \boldsymbol{z}_{a_{t}^{\ast},t}^{\top} \bar{\boldsymbol{\psi}}_{t},
\end{align}
where the last inequality follows from the definition of saturated arms. Moreover, $(\ast\ast)$ follows from (\ref{eq:lem2-hyp2}) and (\ref{eq:lem3-hyp1}).
\end{proof}
%-------------------------------
%-------------------------------> Proof of Theorem 1
\subsubsection{Proof of Theorem \ref{thm:regbound}} % \ref{thm:regbound} 1
\label{app:TheoremOneProof}
\begin{proof}
We start by decomposing the expected dynamic regret as follows. With probability $1 - 2 \exp(- \frac{d \varepsilon^{2}}{8})$, the following holds. %into an expected surrogate regret and a bias term arising from time variation on the true parameter. Therefore, 
\begin{align} \nonumber
\label{eq:decomposedregret}
    \mathbbm{E}[\mathcal{R}_{T}] 
    % &= \mathbbm{E}[\mathcal{R}_{T} I\{E_{\boldsymbol{\psi}^{\ast}}(t)\}] + \mathbbm{E}[\mathcal{R}_{T} I\{\overline{E_{\boldsymbol{\psi}^{\ast}}(t)}\}]
    % \\ \nonumber
    % &= \sum_{t=1}^{T} \mathbbm{E}[(\boldsymbol{x}_{a_{t}^{\ast}, t}^{\top} \boldsymbol{\theta}_{t}^{\ast} - \boldsymbol{x}_{a_{t}, t}^{\top} \boldsymbol{\theta}_{t}^{\ast}) I\{E_{\boldsymbol{\psi}^{\ast}}(t)\}] + 4 T \exp(- \frac{d \varepsilon^{2}}{8}) \\ \nonumber 
    &= \sum_{t=1}^{T} \mathbbm{E}[\boldsymbol{x}_{a_{t}^{\ast}, t}^{\top} \boldsymbol{\theta}_{t}^{\ast} - \boldsymbol{x}_{a_{t}, t}^{\top} \boldsymbol{\theta}_{t}^{\ast}] \\ \nonumber 
    &\stackrel{\text{($\ast$)}}{\leq}
    \sum_{t=1}^{T} \mathbbm{E}[\boldsymbol{z}_{a_{t}^{\ast}, t}^{\top} \boldsymbol{\psi}_{t}^{\ast} - \boldsymbol{z}_{a_{t}, t}^{\top} \boldsymbol{\psi}_{t}^{\ast}] +
    2 T \varepsilon %2 T (\varepsilon + 2 \exp(- \frac{d \varepsilon^{2}}{8})) 
    \\ \nonumber
    &= \sum_{t=1}^{T} \mathbbm{E}[\boldsymbol{z}_{a_{t}^{\ast}, t}^{\top} \boldsymbol{\psi}_{t}^{\ast} - \boldsymbol{z}_{a_{t}, t}^{\top} \boldsymbol{\bar{\psi}}_{t} + \boldsymbol{z}_{a_{t}, t}^{\top} \boldsymbol{\bar{\psi}}_{t} 
    - \boldsymbol{z}_{a_{t}, t}^{\top} \boldsymbol{\psi}_{t}^{\ast}
    + \boldsymbol{z}_{a_{t}^{\ast}, t}^{\top} \boldsymbol{\bar{\psi}}_{t} - \boldsymbol{z}_{a_{t}^{\ast}, t}^{\top} \boldsymbol{\bar{\psi}}_{t}] 
    % \\ \nonumber
    % &
    +2 T \varepsilon 
    %2 T (\varepsilon + 2 \exp(- \frac{d \varepsilon^{2}}{8})) 
    \\ \nonumber
    &\leq \sum_{t=1}^{T} (\mathbbm{E}[\boldsymbol{z}_{a_{t}^{\ast}, t}^{\top} \boldsymbol{\bar{\psi}}_{t} - \boldsymbol{z}_{a_{t}, t}^{\top} \boldsymbol{\bar{\psi}}_{t}] 
    + \sum_{t=1}^{T} \mathbbm{E}[\langle \boldsymbol{z}_{a_{t}^{\ast}, t} - \boldsymbol{z}_{a_{t}, t}, \boldsymbol{\psi}_{t}^{\ast} - \boldsymbol{\bar{\psi}}_{t} \rangle]
    % \\ \nonumber 
    % &
    +2 T \varepsilon %2 T (\varepsilon + 2 \exp(- \frac{d \varepsilon^{2}}{8})) 
    \\
    &\leq \sum_{t=1}^{T} \mathbbm{E}[\boldsymbol{z}_{a_{t}^{\ast}, t}^{\top} \boldsymbol{\bar{\psi}}_{t} - \boldsymbol{z}_{a_{t}, t}^{\top} \boldsymbol{\bar{\psi}}_{t}] 
    + 2 L_{\boldsymbol{z}} \sum_{t=1}^{T} \norm{\boldsymbol{\psi}_{t}^{\ast} - \boldsymbol{\bar{\psi}}_{t}}_{2} +
    2 T \varepsilon, %2 T (\varepsilon
    %+ 2 \exp(- \frac{d \varepsilon^{2}}{8})),
\end{align}
where $(\ast)$ follows from (\ref{eq:orgvsred1}) and (\ref{eq:orgvsred2}). Let $\zeta = 2d\log(\frac{1}{\gamma}) + 2 \frac{d}{T} \log(1 + \frac{L_{\boldsymbol{z}}^{2}}{d\lambda(1-\gamma)})$. The first term in (\ref{eq:decomposedregret}) is bounded as %, i.e., expected surrogate regret,
\begin{align} \nonumber
    \sum_{t=1}^{T} \mathbbm{E}[\boldsymbol{z}_{a_{t}^{\ast}, t}^{\top} \boldsymbol{\bar{\psi}}_{t} 
    - \boldsymbol{z}_{a_{t}, t}^{\top} \boldsymbol{\bar{\psi}}_{t}] 
    &\leq \sum_{t=d+1}^{T} \mathbbm{E}[\langle \boldsymbol{z}_{a_{t}^{\ast}, t} - \boldsymbol{z}_{a_{t}, t}, \boldsymbol{\bar{\psi}}_{t} \rangle I\{E_{\boldsymbol{\hat{\psi}}}\}] 
    + 2 L_{\boldsymbol{\psi}} L_{\boldsymbol{z}} T \mathbbm{P}(\overline{E_{\boldsymbol{\hat{\psi}}}}) + d \\ \nonumber
    &\stackrel{\text{($a$)}}{\leq} (\alpha + \beta) \sum_{t = 1}^{T} (1 + \frac{2}{p_{3} - p_{2}}) \mathbbm{E}_{t}[s_{a_{t},t}] 
    + 2 L_{\boldsymbol{\psi}} L_{\boldsymbol{z}} T (p_{1} + p_{2}) + d \\ \nonumber
    &\stackrel{\text{($b$)}}{\leq} (\alpha + \beta) \sum_{t = 1}^{T} (1 + \frac{2}{p_{3} - p_{2}}) \mathbbm{E}_{t}[\min\{1, s_{a_{t},t}\}]
    + 2 L_{\boldsymbol{\psi}} L_{\boldsymbol{z}} T (p_{1} + p_{2}) + d \\ 
    &\stackrel{\text{($c$)}}{\leq} (\alpha + \beta) (1 + \frac{2}{p_{3} - p_{2}}) \sqrt{\zeta T} + 2 L_{\boldsymbol{\psi}} L_{\boldsymbol{z}} T(p_{1} + p_{2}) + d,
\end{align}
where $(a)$ follows from Lemma \ref{lem:4}. $(b)$ holds due to the fact that the expected dynamic regret is upper bounded by $2T$ and $\alpha+\beta \geq 2$. Moreover, $(c)$ follows from Corollary 4 in \cite{Russac19:WLB}.

The second term in (\ref{eq:decomposedregret}) can be bounded as follows. For any integer $D > 0$, %i.e, the bias term,
\begin{align} \nonumber
    2 L_{\boldsymbol{z}} \sum_{t=1}^{T} \norm{\boldsymbol{\psi}_{t}^{\ast} - \boldsymbol{\bar{\psi}}_{t}}_{2} 
    &= 2 L_{\boldsymbol{z}} \sum_{t=1}^{T} \norm{\boldsymbol{Z}_{t}^{-1} \sum_{\tau = 1}^{t-1} \gamma^{-\tau} \boldsymbol{z}_{a_{\tau}, \tau} \boldsymbol{z}_{a_{\tau}, \tau}^{\top} (\boldsymbol{\psi}_{\tau}^{\ast} - \boldsymbol{\psi}_{t}^{\ast})}_{2} \\ \nonumber
    &\leq 2 L_{\boldsymbol{z}} \sum_{t=1}^{T} \norm{\boldsymbol{Z}_{t}^{-1} \sum_{\tau = t-D}^{t-1} \gamma^{-\tau} \boldsymbol{z}_{a_{\tau}, \tau} \boldsymbol{z}_{a_{\tau}, \tau}^{\top} (\boldsymbol{\psi}_{\tau}^{\ast} - \boldsymbol{\psi}_{t}^{\ast})}_{2} \\ \nonumber
    &\hspace{20mm}+ 2 L_{\boldsymbol{z}} \sum_{t=1}^{T} \norm{\boldsymbol{Z}_{t}^{-1} \sum_{\tau = 1}^{t-D-1} \gamma^{-\tau} \boldsymbol{z}_{a_{\tau}, \tau} \boldsymbol{z}_{a_{\tau}, \tau}^{\top} (\boldsymbol{\psi}_{\tau}^{\ast} - \boldsymbol{\psi}_{t}^{\ast})}_{2} \\ \nonumber
    &\stackrel{\text{$(a)$}}{\leq} 2 L_{\boldsymbol{z}} 
    \sum_{t=1}^{T} \sum_{p=t-D}^{t-1} \norm{\boldsymbol{Z}_{t}^{-1} \sum_{\tau = t-D}^{p} \gamma^{-\tau} \boldsymbol{z}_{a_{\tau}, \tau} \boldsymbol{z}_{a_{\tau}, \tau}^{\top} (\boldsymbol{\psi}_{p}^{\ast} - \boldsymbol{\psi}_{p+1}^{\ast})}_{2} \\ \nonumber
    &\hspace{20mm}+ L_{\boldsymbol{z}} \sum_{t=1}^{T} \frac{2}{\lambda} \norm{\sum_{\tau = 1}^{t-D-1} \gamma^{t-\tau-1} \boldsymbol{z}_{a_{\tau}, \tau} \boldsymbol{z}_{a_{\tau}, \tau}^{\top} (\boldsymbol{\psi}_{\tau}^{\ast} - \boldsymbol{\psi}_{t}^{\ast})}_{2} \\ \nonumber
    &\stackrel{\text{$(b)$}}{\leq} 2 L_{\boldsymbol{z}} \sum_{t=1}^{T} \sum_{p=t-D}^{t-1} \norm{\boldsymbol{\psi}_{p}^{\ast} - \boldsymbol{\psi}_{p+1}^{\ast}}_{2} + \frac{4 L_{\boldsymbol{z}}^{3} L_{\boldsymbol{\psi}}}{\lambda} \frac{\gamma^{D}}{1-\gamma} T \\
    &\leq 2 L_{\boldsymbol{z}} D B_{T} + \frac{4 L_{\boldsymbol{z}}^{3} L_{\boldsymbol{\psi}}}{\lambda} \frac{\gamma^{D}}{1-\gamma} T,
\end{align}
%\leq
where $(a)$ follows by interchanging the order of summations and using $\boldsymbol{Z}_{t}^{-2} \preccurlyeq  (\frac{\gamma^{t-1}}{\lambda})^{2} \boldsymbol{I}_{d \times d}$. Moreover, $(b)$ follows from the fact that, for all $p$ such that $t-D \leq p \leq t-1$, we have $\mu_{\max}(\boldsymbol{Z}_{t}^{-1} \sum_{\tau = t-D}^{p} \gamma^{-\tau} \boldsymbol{z}_{a_{\tau}, \tau} \boldsymbol{z}_{a_{\tau}, \tau}^{\top}) \leq 1$. 

Combining all the above results, with probability $1 - 2 \exp(- \frac{d \varepsilon^{2}}{8})$, the expected dynamic regret is bounded as
\begin{align} \nonumber
    \mathbbm{E}[\mathcal{R}_{T}] 
    &\leq 
    (\alpha + \beta) (1 + \frac{2}{p_{3} - p_{2}}) \sqrt{\zeta T} + 2 L_{\boldsymbol{\psi}} L_{\boldsymbol{z}} T(p_{1} + p_{2}) \\
    &\hspace{50mm}+ d + 2 L_{\boldsymbol{z}} D B_{T} + \frac{4 L_{\boldsymbol{z}}^{3} L_{\boldsymbol{\psi}}}{\lambda} \frac{\gamma^{D}}{1-\gamma} T + 2 T \varepsilon . %2 T (\varepsilon + 2 \exp(- \frac{d \varepsilon^{2}}{8})).
\end{align}
Therefore, by choosing $D = \frac{\log(T)}{1-\gamma}$, we conclude the proof.
\end{proof}
%-------------------------
%-------------------------? Additional Info on Experimental Setup
\subsection{Additional Information on Experimental Setup}
\label{app:AddInfo}
We tune the hyperparameters of each benchmark algorithm by performing a grid search using a user stream as the validation data for $30,000$ time steps. To that end, we ran each algorithm with hyperparameters taken from a grid for three repetitions and chose the parameter values that resulted in the highest averaged accumulated reward. The parameters of those benchmark algorithms that are designed for non-stationary environments, i.e., D-LinTS and D-LinTS-RP, are tuned with introducing change points at time steps $\{5000, 10000, 20000\}$. As mentioned before, DeepFM is not originally designed for the online learning setting. Thus, to make a fair comparison, we use the same change points in the validation dataset for tuning the parameters of the DeepFM algorithm. To tune the parameters of those benchmark algorithms that rely on dimensionality reduction, i.e., D-LinTS-RP and CBRAP, we additionally average the accumulated reward over the different number of reduced dimensions. Thus, for those algorithms, we chose one set of parameters for all reduced dimensions.

%-------------------------> Table
\begin{table}[t!]
\renewcommand{\arraystretch}{1.3}
\renewcommand{\tabcolsep}{1.3mm}
    % { \large
    \begin{center}
    \caption{Tuned hyperparameters of the algorithms in our experiments.}
    % \vspace{2mm}
    \label{Table:parameters}
    % \resizebox{0.48\textwidth}{!}{
    % \resizebox{\columnwidth}{!}{%
    % \scalebox{0.6}{
    \resizebox{0.55\textwidth}{!}{
    \begin{tabular}{|c|c|c|}
        \hline
        \bfseries Dataset & \bfseries Policy & \bfseries Hyperparameters  \\       
        %-------------------------> MovieLens 
        \hline 
        \multirow{10}{*}{MovieLens 10M} & D-LinTS & $\gamma = 0.99, a = 0.1$ \\ 
        \cline{2-3} 
        & LinTS & $\nu = 0.3$ \\
        \cline{2-3} 
        & D-LinTS-RP & $\gamma = 0.99, \xi = 0.1$ \\
        \cline{2-3} 
         & LinUCB & $\alpha = 0.6$ \\
        \cline{2-3} 
        & CBRAP & $\alpha = 0.6$ \\
        \cline{2-3} 
        & $\epsilon$-greedy & $\varepsilon = 0.001$ \\
        \cline{2-3} 
        &\multirow{4}{*}{DeepFM} & Activation function = ReLU, \\ 
        & & Dropout probability = 0.2, \\
        & & Learning rate = 0.01, \\
        & & Hidden units = [1024, 256, 128] \\
        \hline 
        %-------------------------> Jester 
        % \hline 
        \multirow{10}{*}{Jester} & D-LinTS & $\gamma = 0.99, a = 0.1$ \\ 
        \cline{2-3} 
        & LinTS & $\nu = 0.3$ \\
        \cline{2-3} 
        & D-LinTS-RP & $\gamma = 0.99, \xi = 0.1$ \\
        \cline{2-3} 
         & LinUCB & $\alpha = 0.3$ \\
        \cline{2-3} 
        & CBRAP & $\alpha = 0.3$ \\
        \cline{2-3} 
        & $\epsilon$-greedy & $\varepsilon = 0.05$ \\
        \cline{2-3} 
        &\multirow{4}{*}{DeepFM} & Activation function = ReLU, \\ 
        & & Dropout probability = 0.2, \\
        & & Learning rate = 0.01, \\
        & & Hidden units = [1024, 256, 128] \\
        \hline 
        %-------------------------> Amazon Books 
        % \hline 
        \multirow{10}{*}{Amazon Books} & D-LinTS & $\gamma = 0.99, a = 0.1$ \\ 
        \cline{2-3} 
        & LinTS & $\nu = 0.1$ \\
        \cline{2-3} 
        & D-LinTS-RP & $\gamma = 0.99, \xi = 0.1$ \\
        \cline{2-3} 
         & LinUCB & $\alpha = 0.2$ \\
        \cline{2-3} 
        & CBRAP & $\alpha = 0.3$ \\
        \cline{2-3} 
        & $\epsilon$-greedy & $\varepsilon = 1e-06$ \\
        \cline{2-3} 
        &\multirow{4}{*}{DeepFM} & Activation function = ReLU, \\ 
        & & Dropout probability = 0.5, \\
        & & Learning rate = 0.001, \\
        & & Hidden units = [1024, 1024, 1024] \\
        \hline 
    \end{tabular}
    }
    % }
    \end{center}
\end{table}
%-------------------------

We simultaneously tuned the parameters $\gamma$ and $a$ for D-LinTS by performing a grid search over the sets $\{0.9999, 0.999, 0.99, 0.9, 0.8, 0.7, 0.6\}$ and $\{0.1, 0.2, 0.3, 0.4, 0.5, 0.6, 0.7, 0.8, 0.9\}$, respectively. For LinTS, we used the grid $\{0.1, 0.2, 0.3, 0.4, 0.5, 0.6, 0.7, 0.8, 0.9\}$ to tune the hyperparameter $\nu$. For D-LinTS-RP, we chose the hyperparameters $\gamma$ and $\xi$ from the sets $\{0.9999, 0.999, 0.99, 0.9, 0.8, 0.7, 0.6\}$ and $\{0.1, 0.2, 0.3, 0.4, 0.5, 0.6, 0.7, 0.8, 0.9\}$, respectively. We set $\kappa^2 = \frac{1}{n}$ when generating the random projection matrix. For LinUCB, we chose the hyperparameter $\alpha$ from the set $\{0.1, 0.2, 0.3, 0.4, 0.5, 0.6, 0.7, 0.8, 0.9\}$. For CBRAP, we chose the hyperparameter $\alpha$ from the set $\{0.1, 0.2, 0.3, 0.4, 0.5, 0.6, 0.7, 0.8, 0.9\}$ and set $\kappa^2 = \frac{1}{n}$ when generating the random projection matrix. For $\epsilon$-greedy, we chose the hyperparameter $\epsilon$ from the set $\{0.000001, 0.000005, 0.00001, 0.00005, 0.0001, 0.0005, 0.001, 0.005, 0.01, 0.05, 0.1, 0.2, 0.3, 0.4, \newline 0.5, 0.6, 0.7, 0.8, 0.9\}$. For DeepFM, we followed the approach proposed in \cite{Guo17:DAF} and chose hyperparameters that influenced the performance of the model the most. Hence, we considered several options for each component and hyperparameter of the model. More precisely, we chose the possible activation functions from $\{\textup{ReLU}, \textup{Tanh}\}$, the possible dropout probability from $\{0.1, 0.2, 0.3, 0.4, 0.5\}$, the learning rate from $\{0.01, 0.001, 0.0001\}$, and the number of hidden units in each of the three layers in the deep part of the model from the set $\{128, 256, 512, 1024\}$. \textbf{Table \ref{Table:parameters}} summarizes the final selected parameters for our experiments with the three datasets.
% (1/d)

%-------------------------> Additional Info on Experimental Setup
\subsection{Additional Information and Results regarding the Numerical Experiments}
\label{app:AddExps}
\textbf{Table \ref{Table:results-full}} lists the average runtime, the average cumulative reward, and the average cumulative NDCG@5 of each policy corresponding to different datasets. For D-LinTS-RP and CBRAP, we list the results for different reduced dimensions. We reduce the context dimension to $5\%$, $10\%$, $20\%$, and $50\%$ of the original context dimension. All the policies are evaluated on a single compute node with 64 Intel Xeon Gold 6226R CPUs and 64G of RAM.
% 
%-------------------------> Table
\begin{table*}[!t]
\renewcommand{\arraystretch}{1.2}
\renewcommand{\tabcolsep}{1.2mm}
    \begin{center}
    \caption{Comparison of cumulative reward, cumulative NDCG@5, and time consumption of different policies corresponding to different datasets and context dimensions. The reported values are averaged over five repetitions.}
    % \vspace{2mm}
    \label{Table:results-full}
    \resizebox{0.94\textwidth}{!}{
    \begin{tabular}{c|c|c|c|c|c}
        \hline
        \bfseries Dataset & \bfseries Policy & \bfseries Context Dimension & \bfseries Runtime (second) & \bfseries Cumulative Reward  & \bfseries Cumulative NDCG@5 \\
        \hline 
        %-------------------------> MovieLens
        \multirow{15}{*}{MovieLens 10M} & D-LinTS & 120 & 1739.8 & 74771.6 & 48499.1 \\ 
        \cline{2-6} 
        & LinTS & 120 & 1370.9 & 53814.0 & 33572.3 \\
        \cline{2-6} 
        % \cline{3-5}
        & \multirow{5}{*}{D-LinTS-RP} & 6 & 1027.4 & 59379.8 & 33934.6 \\
        \cline{3-6} 
        &  & 12 & 1054.9 & 66628.6 & 37764.2 \\
        \cline{3-6} 
        & & 24 & 1109.2 & 70130.0 & 39509.2 \\
        \cline{3-6} 
        & & 60 & 1231.5 & 74294.6 & 44014.6 \\
        \cline{3-6} 
        & & 120 & 1753.4 & 74379.8 & 48048.1 \\
        % \cline{2-6} 
        %  & LinUCB & 120 & 2452.2 & 58130.0 & 33378.2 \\
        \cline{2-6} 
        & \multirow{5}{*}{CBRAP} & 6 & 1647.7 & 38660.0 & 23687.6 \\
        \cline{3-6}
        & & 12 & 1645.8 & 42485.0 & 25662.7 \\
        \cline{3-6} 
        & & 24 & 1917.0 & 57314.0 & 32521.0 \\
        \cline{3-6} 
        & & 60 & 1989.3 & 59425.0 & 34382.5 \\
        \cline{3-6} 
        & & 120 & 2440.3 & 61445.0 & 36496.0\\
        \cline{2-6} 
         & DeepFM & 120 & 4174.4 & 25419.8 & 23624.2 \\
        \cline{2-6} 
        & $\epsilon$-greedy & -- & 686.7 & 33066.4 & 22866.9 \\
        \cline{2-6} 
        & Random & -- & 857.8 & 20683.4 & 20634.6 \\
        %-------------------------> Jester
        \hline 
        \multirow{15}{*}{Jester} & D-LinTS & 300 & 6800.4 & 44134.6 & 33895.4  \\
        \cline{2-6} 
        & LinTS & 300 & 2769.7 & 36695.2 & 30340.7 \\
        \cline{2-6} 
        & \multirow{5}{*}{D-LinTS-RP} 
        & 15 & 596.0 & 34681.8 & 25716.7  \\
        \cline{3-6} 
        &  & 30 & 438.4 & 41278.6 & 32077.9 \\
        \cline{3-6} 
        &  & 60 & 610.6 & 43382.8 & 32116.6 \\
        \cline{3-6} 
        &  & 150 & 1554.3 & 43693.2 & 33030.9 \\
        \cline{3-6} 
        &  & 300 & 6362.9 & 43747.8 & 34022.0 \\
        % \cline{2-6} 
        % & LinUCB & 300 & 798.8 & 30426.0 & 25540.9 \\
        \cline{2-6} 
        & \multirow{5}{*}{CBRAP} 
        & 15 & 450.9 & 18159.0 & 20113.8  \\
        \cline{3-6} 
        &  & 30 & 475.8 & 19918.0 & 18023.0 \\
        \cline{3-6} 
        &  & 60 & 521.3 & 26171.0 & 22018.0 \\
        \cline{3-6} 
        &  & 150 & 673.6 & 25081.0 & 22004.3 \\
        \cline{3-6} 
        &  & 300 & 841.0 & 27740.8 & 23026.6 \\
        \cline{2-6} 
         & DeepFM & 300 & 3912.7 & 19275.0 & 22513.9 \\
        \cline{2-6}
        & $\epsilon$-greedy & -- & 223.7 & 29832.4 & 27273.1 \\
        \cline{2-6} 
        & Random & -- & 235.5 & 13929.0 & 18179.2 \\
         %-------------------------> Amazon Books
        \hline 
        \multirow{15}{*}{Amazon Books} & D-LinTS & 200 & 2549.3 & 35018.2 & 12929.0 \\
        \cline{2-6} 
        & LinTS & 200 & 1315.6 & 9513.6 & 4594.4\\
        \cline{2-6} 
        & \multirow{5}{*}{D-LinTS-RP} 
        & 10 & 505.8 & 20989.6 & 8057.9  \\
        \cline{3-6} 
        &  & 20 & 537.8 & 26544.0 & 9694.7 \\
        \cline{3-6} 
        &  & 40 & 586.2 & 34118.0 & 12257.4 \\
        \cline{3-6} 
        &  & 100 & 973.2 & 34654.0 & 12603.2 \\
        \cline{3-6} 
        &  & 200 & 2568.8 & 34788.6 & 12995.2 \\
        % \cline{2-6} 
        % & LinUCB & 200 & 1131.1 & 5257.0 & 2914.6 \\
        \cline{2-6} 
        & \multirow{5}{*}{CBRAP} 
        & 10 & 742.2 & 1901.0 & 1894.8  \\
        \cline{3-6} 
        &  & 20 & 808.4 & 8618.0 & 3369.9 \\
        \cline{3-6} 
        &  & 40 & 833.8 & 12500.0 & 4948.5 \\
        \cline{3-6} 
        &  & 100 & 918.5 & 5779.0 & 3073.5 \\
        \cline{3-6} 
        &  & 200 & 1281.2 & 7549.0 & 3263.6 \\
        \cline{2-6}
         & DeepFM & 200 & 4377.9 & 5562.8 & 3402.3 \\
        \cline{2-6}
        & $\epsilon$-greedy & -- & 296.9 & 1102.0 & 1454.8 \\
        \cline{2-6} 
        & Random & -- & 361.2 & 1625.2 & 1625.4 \\
        \hline
    \end{tabular}
    }
    \end{center}
\end{table*}
\newpage
%-------------------------------------> References
\bibliographystyle{IEEEbib}
\bibliography{references}
%-------------------------------------
% \newpage
\end{document}